\documentclass[10pt,journal,compsoc]{IEEEtran}
\ifCLASSOPTIONcompsoc
  \usepackage[nocompress]{cite}
\else
  \usepackage{cite}
\fi
\ifCLASSINFOpdf
\else
\fi

\usepackage{mathrsfs}
\usepackage{amsmath}
\usepackage{bm}
\usepackage{amsfonts}
\usepackage{amssymb}
\usepackage{color}
\usepackage{latexsym}
\usepackage{hyperref}
\usepackage{multirow}
\usepackage{float}
\floatstyle{ruled}
\newfloat{algorithm}{hbt}{alg}[section]
\floatname{algorithm}{Algorithm}
\floatstyle{boxed}
\newfloat{program}
\usepackage{amsthm}
\usepackage{cases}
\usepackage{color}
\usepackage{graphicx}
\usepackage{subfig}

\newtheorem{theorem}{Theorem}[section]
\newtheorem{lemma}[theorem]{Lemma}

\newtheorem{remark}[theorem]{Remark}
\newtheorem{definition}[theorem]{Definition}
\newenvironment{proof}{\noindent{\em Proof:}}{\quad \hfill$\Box$\vspace{2ex}}
\numberwithin{equation}{section}
\renewcommand{\thefigure}{\thesection.\arabic{figure}}
\renewcommand{\thetable}{\thesection.\arabic{table}}

\def \bK {{\boldsymbol K}}
\def \bx {{\boldsymbol x}}
\def \ba {{\boldsymbol a}}
\def \bp {{\boldsymbol p}}
\def \by {{\boldsymbol y}}
\def \bc {{\boldsymbol c}}
\def \bd {{\boldsymbol d}}

\def \bz {{\boldsymbol z}}

\def \bw {{\boldsymbol w}}
\def \br {{\boldsymbol r}}
\def \bu {{\boldsymbol u}}
\def \bv {{\boldsymbol v}}

\def \cD {{\cal D}}

\def \bN {\mathbb N}
\def \bR {\mathbb R}
\def \cF {{\cal F}}
\def \cR {{\cal R}}

\def \supp {\,{\rm supp}\,}
\def \diag {\,{\rm diag}\,}

\begin{document}
%
% paper title
% Titles are generally capitalized except for words such as a, an, and, as,
% at, but, by, for, in, nor, of, on, or, the, to and up, which are usually
% not capitalized unless they are the first or last word of the title.
% Linebreaks \\ can be used within to get better formatting as desired.
% Do not put math or special symbols in the title.
\title{Kernel Support Vector Machine Classifiers with the \texorpdfstring{$\ell_0$}{}-Norm Hinge Loss}
%
%
% author names and IEEE memberships
% note positions of commas and nonbreaking spaces ( ~ ) LaTeX will not break
% a structure at a ~ so this keeps an author's name from being broken across
% two lines.
% use \thanks{} to gain access to the first footnote area
% a separate \thanks must be used for each paragraph as LaTeX2e's \thanks
% was not built to handle multiple paragraphs
%
%
%\IEEEcompsocitemizethanks is a special \thanks that produces the bulleted
% lists the Computer Society journals use for "first footnote" author
% affiliations. Use \IEEEcompsocthanksitem which works much like \item
% for each affiliation group. When not in compsoc mode,
% \IEEEcompsocitemizethanks becomes like \thanks and
% \IEEEcompsocthanksitem becomes a line break with idention. This
% facilitates dual compilation, although admittedly the differences in the
% desired content of \author between the different types of papers makes a
% one-size-fits-all approach a daunting prospect. For instance, compsoc
% journal papers have the author affiliations above the "Manuscript
% received ..."  text while in non-compsoc journals this is reversed. Sigh.

\author{Rongrong~Lin,
        Yingjia~Yao, and~ Yulan~Liu
\IEEEcompsocitemizethanks{\IEEEcompsocthanksitem R. Lin, Y. Yao and Y. Liu (corresponding author) are with School of Mathematics and Statistics, Guangdong University of Technology, Guangzhou 510520, P.R. China (Email: linrr@gdut.edu.cn, 2112114020@mail2.gdut.edu.cn, ylliu@gdut.edu.cn.}}

\IEEEtitleabstractindextext{%
\begin{abstract}
   Support Vector Machine (SVM) has been one of the most successful machine learning techniques for binary classification problems. The key idea is to maximize the margin from the data to the hyperplane subject to correct classification on training samples. The commonly used hinge loss and its variations are sensitive to label noise, and unstable for resampling due to its unboundedness.
   This paper is concentrated on the kernel SVM with the $\ell_0$-norm hinge loss (referred as $\ell_0$-KSVM), which is a composite  function of hinge loss and $\ell_0$-norm and then could overcome the difficulties mentioned above. In consideration of the nonconvexity and nonsmoothness of $\ell_0$-norm hinge loss, we first characterize the limiting subdifferential of the $\ell_0$-norm hinge loss and then derive  the equivalent relationship among the proximal stationary point, the Karush-Kuhn-Tucker point, and the local optimal solution of $\ell_0$-KSVM.  Secondly, we develop an ADMM algorithm for $\ell_0$-KSVM, and obtain that any limit point of the sequence generated by the proposed algorithm is a locally optimal solution. Lastly, some experiments on the synthetic and real datasets are illuminated to show that $\ell_0$-KSVM can achieve comparable accuracy compared with the standard KSVM while the former generally enjoys fewer support vectors.
\end{abstract}

% Note that keywords are not normally used for peerreview papers.
}

\maketitle

% To allow for easy dual compilation without having to reenter the
% abstract/keywords data, the \IEEEtitleabstractindextext text will
% not be used in maketitle, but will appear (i.e., to be "transported")
% here as \IEEEdisplaynontitleabstractindextext when the compsoc
% or transmag modes are not selected <OR> if conference mode is selected
% - because all conference papers position the abstract like regular
% papers do.
\IEEEdisplaynontitleabstractindextext
% \IEEEdisplaynontitleabstractindextext has no effect when using
% compsoc or transmag under a non-conference mode.

% For peer review papers, you can put extra information on the cover
% page as needed:
% \ifCLASSOPTIONpeerreview
% \begin{center} \bfseries EDICS Category: 3-BBND \end{center}
% \fi
%
% For peerreview papers, this IEEEtran command inserts a page break and
% creates the second title. It will be ignored for other modes.
\IEEEpeerreviewmaketitle

\IEEEraisesectionheading{\section{Introduction}\label{sec:introduction}}

     \IEEEPARstart{S}{upport} Vector Machine (SVM) originally proposed by Vapnik in 1982 is one of the most successful methods for classification problems in machine learning \cite{Scholkopf2001,Steinwart08,Zaki2020}.
      To be precise, a labelled training dataset $\cD:=\{(\bx_i,y_i):i\in\bN_m\}\subseteq \bR^d\times \{-1,1\}$ is given,
      where $\bN_m:=\{1,2\ldots,m\}$.
      To find a separating hyperplane $h(\bx)=\bw^{\top}\bx+b$ with $\bw\in\bR^d$ and $b\in\bR$ to correctly classify the labelled examples, the standard soft-margin linear SVM with respect to the training data $\cD$ is established as follows:
     $$
     \min_{\bw\in\bR^d,b\in \mathbb{R}}\frac{1}{2}\|\bw\|_2^2+C\sum_{i=1}^m\Big[\big(1-y_i(\bw^{\top}\bx_i+b)\big)_+\Big]^k
     $$
     where $C>0$, $t_+\!:=\!\max\{t,0\}$ for any $t\!\in \!\mathbb{R}$ and $k\!=\!1{\,\rm or\,}2$.
     In the above SVM model, the second term $[(1-y_i(\bw^{\top}\bx_i+b))_+]^k$ is known as the hinge loss ($k=1$)
     or the squared hinge loss ($k=2$), and the first term $\frac{1}{2}\|\bw\|_2^2$ is to maximize the margin (see, e.g., \cite[Chapter 12]{Scholkopf2001}).

     Thanks to the kernel tricks, the linear SVM can be easily extended to build up nonlinear classifiers \cite{Steinwart08,Zaki2020}.
     The key idea is to map the original $d$-dimensional points in the input space to points in a high-dimensional feature space $\cF$ via some nonlinear transformation $\phi:\bR^d\to\cF$. In general, the feature space $\cF$ is assumed to be a Hilbert space with the inner product $\langle \cdot,\cdot\rangle_{\cF}$.
     The hyperplane in feature space $\cF$ with the weight $\bw\in\cF$ and the bias $b\in\bR$ takes the form
     $h(\boldsymbol{x})=\langle \boldsymbol{w},\phi(\boldsymbol{x})\rangle_{\cF}+b$.
     As a result, the nonlinear SVM is built as follows:
     \begin{equation}\label{SHSVM}
     \min_{\bw\in\cF,b\in \mathbb{R}}\frac{1}{2}\|\bw\|_{\cF}^2+C\sum_{i=1}^m\Big[\big(1-y_i(\langle \boldsymbol{w},\phi(\boldsymbol{x}_i)\rangle_{\cF}+b)\big)_+\Big]^{k}.
     \end{equation}
     Generally, the dimensionality of $\cF$ is very high, the curse of dimensionality happens if
     the term $\langle \boldsymbol{w},\phi(\boldsymbol{x}_i)\rangle_{\cF}$ in \eqref{SHSVM} is directly computed.
     To tackle this difficulty, the kernel method is proposed.
     A function $K:\bR^d\times \bR^d\to\bR$ with the feature map $\phi$ defined by
     \begin{equation}\label{kernel}
     K(\bz,\bz'):=\langle\phi(\bz),\phi(\bz')\rangle_{\cF},\mbox{ for any } \bz,\bz'\in\bR^d,
     \end{equation}
     is a positive definite kernel \cite{Scholkopf2001}.
     In this sense, the positive definite kernel is an extension of the standard inner product in a Euclidean space.

    The hinge loss used in the standard nonlinear SVM \eqref{SHSVM} is related to the shortest distance between positive labelled data and negative labelled data. Hence, the corresponding classifier is sensitive to noise and outliers \cite{Frenay2014}, and unstable for resampling. To solve this, a large body of convex or nonconvex loss functions is discussed in the literature (see, e.g., \cite{Huang2014,Feng2016,Wang2022,WangXiu2022} and the references therein).  
    The SVM with the pinball loss and the truncated pinball loss were considered in \cite{Huang2014} and \cite{Shen2017}, respectively.
    A class of nonconvex and smooth  margin-based classification losses were discussed in \cite{Feng2016}. Two typical examples are $\ell_{\sigma}(t)=\sigma^2(1-e^{-t_+^2/\sigma^2})$ and $\ell_{\sigma}(t)=\sigma^2\log(1+t_+^2/\sigma^2)$.
    The SVM with ramp loss \cite{WangShao2022}, truncated Huber loss \cite{WangShao2023}, truncated smoothly clipped absolute deviation (SCAD) loss function\cite{Wang2023}, and truncated least squares loss \cite{WangLi2023} are systematically studied by Wang and his collaborators.

     The binary loss is certainly an ideal loss function for the binary classification problem \cite{Chen1996,Cortes1995,Domingos1997,Brooks2011,Feng2016,Nguyen2013}. Specifically, we count the misclassification error if a sample is classified wrongly we incur a loss $1$, otherwise there is no penalty. Based on the binary loss and the hinge loss, the $\ell_0$-norm hinge loss is proposed in \cite{Tang2018,Brooks2011, Wang2022} as follows:
    \begin{align*}
    \ell_{0/1}(t):=\|t_+\|_0:=\left\{\begin{array}{cl}
    1,& {\rm if\;} t>0\\
    0,& {\rm if \;}t\le 0
    \end{array}\right..
    \end{align*}
    Here $\|\bx\|_0$ denotes the total number of non-zero elements in the vector $\bx\in\bR^d$.

    In this paper, we study the kernel/nonlinear SVM classifier with the $\ell_0$-norm hinge loss $\ell_{0/1}$ with respect to training data $\cD$ taking the form:
     \begin{equation}\label{SVM01}
     \min_{\boldsymbol{w}\in \cF,b\in \bR}\cR(\boldsymbol{w},b)
     :=\frac{1}{2}\|\boldsymbol{w}\|_{\cF}^2+C\sum_{i=1}^m\ell_{0/1}\big(1-y_i h(\bx_i)\big)
     \end{equation}
    where $C>0$ is a given penalty parameter, and $\phi:\bR^d\to\cF$ is given by \eqref{kernel}. Such a model \eqref{SVM01} is
    referred as \textsf{$\ell_0$-KSVM} throughout this paper.
    In particular, \eqref{SVM01} with $\phi(\bx)=\bx$, i.e.,
     the linear kernel $K(\bx,\bx')=\bx^{\top}\bx'$,  is called \textsf{$\ell_0$-SVM}.

    Unfortunately, $\ell_{0/1}$ loss is nonconvex and nonsmooth.
    To the best of our knowledge, there is very little work on the theoretical and algorithmic analysis for \eqref{SVM01}. Brooks developed in \cite[Section 3.2]{Brooks2011} the universal consistency of $\ell_0$-KSVM.
    Wang et al. in \cite{Wang2022} considered $\ell_0$-SVM.
    A proximal stationary point is proposed and the relations among the proximal stationary point, the global optimal solution and the local optimal solution are built under some conditions (see \cite[Theorem 3.2]{Wang2022}), based on the characterization of the proximal operator of the $\ell_0$-norm hinge loss function. Theorem 3.2 in \cite{Wang2022} proved that a proximal stationary point
    for $\ell_0$-SVM is a locally optimal solution. 
    However, in the converse direction, they only  obtain that any global optimal solution   is also a proximal stationary point under a very strong condition.
    Tang et al. \cite{Tang2018} developed the penalty method for both $\ell_0$-SVM and $\ell_0$-KSVM. The convergence analysis \cite[Theorem 2]{Tang2018} was only considered for $\ell_0$-SVM under the existence of the global minimizer for the penalty method, which is not easily checked.
     In our paper, we first characterize the limiting subdifferentials of the $\ell_0$-norm hinge loss and the 
     composition function with the $\ell_0$-norm hinge loss function and a linear mapping. Based on the limiting subdifferential of the 
     composition function, the definition of Karush-Kuhn-Tucker (KKT) point is given  for the $\ell_0$-KSVM. And then we build the equivalence among the proximal stationary point, Karush-Kuhn-Tucker point, and local optimal solution by the explicit expression of the limiting subdifferential of the $\ell_0$-norm hinge loss function, and obtain that any limit point of the sequence generated by the ADMM algorithm is a locally optimal point.

    The rest of this paper is organized as follows. In Section 2,
    we develop the representer theorem for the solution of $\ell_0$-KSVM.
    For the theoretical and algorithmic analysis, we reformulate $\ell_0$-KSVM in the matrix form.
    As a byproduct, a sufficient condition for the globally optimal solution is given with the help of the Weierstrass Theorem. In Section 3, we are devoted to the optimality of \textsf{$\ell_0$-KSVM}.
    Toward this purpose, firstly, the limiting subdifferentials of $L_{0/1}(g(z))$ defined by \eqref{L01norm} with $g:\mathbb{R}^{l}\to \mathbb{R}^{m}$ being a differentiable map is characterized in Lemma \ref{subdiff-Cgz}. In particular, the explicit expression of limiting subdifferentials for $\ell_0$-norm hinge loss is derived in Theorem \ref{SubDifL01Th2}.
    Secondly, the equivalence among the proximal stationary point, Karush-Kuhn-Tucker point, and local optimal solution is proved in Theorems \ref{KKTTheorm} and \ref{OptiTh1}.
    In Section 4, the ADMM algorithm for $\ell_0$-KSVM is developed, and any limit point of the sequence generated by the proposed algorithm is a proximal stationary point given in Theorem \ref{ConvThm}. In Section 5, we conduct experiments on synthetic and real datasets to show that $\ell_0$-KSVM can achieve comparable accuracy while the former generally has fewer support vectors.  The concluding remarks are made in the last section.

\section{The representer theorem}
In this paper, we shall study the $\ell_0$-KSVM model \ref{SVM01}.
Notice that the term $\langle \boldsymbol{w},\phi(\boldsymbol{x}_i)\rangle_{\cF}$
    in \eqref{SVM01} can not be directly handled once $\cF$ is a high-dimensional or even infinite-dimensional space.  To solve so, we should establish the representer theorem for the optimal solution of \eqref{SVM01}.

    \begin{theorem}[Representer Theorem]\label{RepTh}
     For any solution $(\boldsymbol{w}^*,b)$ of \textsf{$\ell_0$-KSVM} \eqref{SVM01}, there exist constants $c_i$, $i=1,2,\dots,m$ such that the weight vector
    \begin{equation}\label{WReTh}
    \boldsymbol{w}^*=\sum_{i=1}^m c_i\phi(\boldsymbol{x}_i).
    \end{equation}
    \end{theorem}

    \begin{proof} The orthogonal decomposition theorem tells us that every closed subspace of a Hilbert space has a complement.
     Denote by
    \[
    V:={\rm span}\big\{\phi(\boldsymbol{x}_1),\phi(\boldsymbol{x}_2),\dots,\phi(\boldsymbol{x}_m)\big\}
    \]
    the space of the linear span of feature points $\phi(\bx_1),\phi(\bx_2),\dots,\phi(\bx_m)$, and denote $V^{\perp}$ as the orthogonal complement of $V$ in feature space $\cF$.
    For any $\bw\in\cF$, there exist $\boldsymbol{w}^{\parallel}\in V$ and $\boldsymbol{w}^{\perp}\in V^{\perp}$ such that
    \[
    \boldsymbol{w}^*=\boldsymbol{w}^{\parallel}+\boldsymbol{w}^{\perp}.
    \]
    It follows that
    \[
    \langle\boldsymbol{w}^{\perp},\phi(\boldsymbol{x}_i)\rangle_{\cF}=0,\mbox{ for any } i\in\bN_m
    \]
    and then
    $$
    \begin{array}{ll}
    &\displaystyle{\cR(\boldsymbol{w}^*,b)}
    =\displaystyle{\frac{1}{2}\|\boldsymbol{w}^*\|^2_{\cF}
    +C\sum_{i=1}^m\ell_{0/1}(1-y_i h(\boldsymbol{w}^*)) }\\
    &\displaystyle{=\frac{1}{2}\|\boldsymbol{w}^{\parallel}+\boldsymbol{w}^{\perp}\|^2_{\cF}+C\sum_{i=1}^m\ell_{0/1}(1-y_i(\langle\boldsymbol{w}^{\parallel}
    +\boldsymbol{w}^{\perp},\phi(\boldsymbol{x}_i)\rangle_{\cF}+b)) }\\
    &\displaystyle{= \frac{1}{2}\|\boldsymbol{w}^{\parallel}\|^2_{\cF}+ \frac{1}{2}\|\boldsymbol{w}^{\perp}\|^2_{\cF}+C\sum_{i=1}^m\ell_{0/1}(1-y_i(\langle\boldsymbol{w}^{\parallel},\phi(\boldsymbol{x}_i)\rangle_{\cF}+b)) }\\
    &\displaystyle{=\cR(\boldsymbol{w}^{\parallel},b)+ \frac{1}{2}\|\boldsymbol{w}^{\perp}\|^2_{\cF}}.
    \end{array}
    $$
    Note that $\cR(\boldsymbol{w}^*,b)\leq \cR(\boldsymbol{w}^{\parallel},b)$
    since $(\boldsymbol{w}^*,b)$ is a solution of \eqref{SVM01}.
    Consequently,  $\frac{1}{2}\|\boldsymbol{w}^{\perp}\|^2_{\cF}=0$ and then $\bw^{\perp}=0$.
    Hence, $\bw^*=\bw^{\parallel}\in V$, that is,
    there exist constants $c_i$, $i=1,2,\dots,m$ such that \eqref{WReTh} holds.
    \end{proof}

     The above representer theorem asserts that the minimizer of \eqref{SVM01} is a linear combination of the given samples.
     From the proof for Theorem \ref{RepTh}, it is clear that the representer theorem even holds for a general loss function.

    For the convenience of our subsequent analysis, we should rewrite \eqref{SVM01} in the matrix form.
    For any positive definite kernel $K$ defined by \eqref{kernel},
    the kernel matrix $\bK$ with respect to $\cD$ is defined by
    $$
    \bK:=[K(\bx_i,\bx_j):i,j\in\bN_m].
    $$
    Denote the label vector $\by\!:=\!(y_1,y_2,\ldots,y_m)^{\top}\!\in\! \bR^m$ with $y_i\in\{-1,1\}$ for any $i\in\bN_m$, ${\bf 1}\in\bR^m$ a vector with all components being $1$,
    $\bu_{+} :=\big((u_1)_{+},(u_2)_+,\ldots,(u_m)_{+}\big)^{\top}$ for any $\bu\in \bR^m$,
    and
    \begin{equation}\label{L01norm}
    L_{0/1}(\bu):=\sum\limits_{i=1}^{m}\ell_{0/1}(u_i)=\|\bu_+\|_0, \mbox{  for any }\bu\in \bR^m.
    \end{equation}
    By \eqref{WReTh} and \eqref{kernel}, we have
    $$
    \|\bw^*\|_{\cF}^2
    =\sum_{i=1}^m\sum_{j=1}^mc_ic_j\langle \phi(\boldsymbol{x}_i),\phi(\boldsymbol{x}_j)\rangle_{\cF}
    =\bc^{\top}\bK \bc.
    $$
   Hence, $\ell_0$-KSVM (\ref{SVM01}) reduces to
    \begin{align}\label{SVM01eq1}
    \min_{\bc\in\bR^m, b\in \bR}\!\frac{1}{2}\bc^{\top}\bK\bc+C\|({\bf 1}-\diag(\by)\bK\bc-b\by)_+\|_0
   \end{align}
   or the following problem with an extra variable $\bu$,
    \begin{align}\label{SVM01eq2}
    &\min_{\bc\in\bR^m,b\in\bR,\bu\in \bR^m}\frac{1}{2}\bc^{\top}\bK\bc+CL_{0/1}(\bu)\\
    & \qquad\quad {\rm s.t. \quad} \bu+{\rm diag}(\by)\bK \bc+b\by ={\bf 1}, \nonumber
    \end{align}
    where $L_{0/1}$ is defined by \eqref{L01norm}.

    Next, an existence condition for solution of \eqref{SVM01eq1} is given when  $K$ is strictly positive definite on $\bR^{d}$, that is,
    $K$ is a positive definite kernel and for all $n\in\bN$ and $\bz_1,\bz_2,\dots,\bz_n\in\bR^d$ the quadratic form $\sum_{j=1}^n\sum_{i=1}^na_iK(\bz_i,\bz_j)a_j=0$ if and only if
    $\ba:=(a_1,a_2,\ldots,a_n)^{\top}={\bf 0}$.

    \begin{theorem}
      Suppose that $K$ is a strictly positive definite kernel on $\bR^d$, $b\in [-M, M]$ with
      $0<M<\infty$. Then the  optimal solution  to \eqref{SVM01eq1} exists and the solution set is nonempty and compact.
    \end{theorem}
    \begin{proof}
     Denote
     \[
     \Psi(\bc,b)\!:=\!\frac{1}{2}\bc^{\top}\bK\bc\!+\!C\|({\bf 1}\!-\!\diag(\by)\bK\bc-b\by)_+\|_0.
     \]
     Notice that
     \[
     \min_{\bc\in\bR^{m},b\in\bR}\Psi (\bc,b)\le \Psi({\bf{0}},b)\leq Cm<\infty.
     \]
     Denote
     \[
     S:=\{(\bc,b)\in \bR^{m}\times [-M,M] : \Psi(\bc,b)\leq Cm\}.
     \]
     Clearly, $S$ is nonempty.
     Since for any $(\bc,b)\in S$,
     \[
     \frac{1}{2}\lambda_{\rm min}(\bK)\|\bc\|^2\leq  \frac{1}{2} {\bc}^{\top}\bK \bc\leq  \Psi(\bc, b)\leq Cm,
     \]
     we can conclude $\bc$ is bounded and then $S$ is also bounded.
     Hence, the globally optimal solution  to \eqref{SVM01eq1} exists and the solution set is nonempty and compact by the Weierstrass Theorem \cite[Theorem 1.14]{Dhara2011}.
    \end{proof}
    
    For the linear SVM, we have $\bw=\sum_{i=1}^m c_i \bx_i$ in Theorem \ref{WReTh}.
    The existence of the optimal solution $(\bw,b)$ to $\ell_0$-KSVM \eqref{SVM01} with the linear kernel was proved in \cite[Theorem 3.1]{Wang2022}. 
    We should point out that the strict positive definiteness condition imposed on the kernel $K$ on $\bR^d$ is a very mild condition for a class of translation-invariant kernels. A continuous function $\Phi:\bR^d\to\bR$ is called (strictly) positive definite on $\bR^d$ if $K(\bz,\bz')=\Phi(\bz-\bz')$ is (strictly) positive definite on $\bR^d$.
    The positive definiteness of translation-invariant kernels can be characterized by the celebrated Bochner theorem in \cite{Wendland2005}.
    A continuous function $\Phi:\bR^d\to\bR$ is positive definite if
    and only if it is the Fourier transform of a finite nonnegative Borel measure $\mu$ on $\bR^d$ \cite[Theorem 6.6]{Wendland2005}, that is,
    $$
    \Phi(\bz)=\int_{\bR^d}e^{-\text{i}\bz^{\top}\xi}d\mu(\xi),\ \bz\in\bR^d
    $$
    where $\text{i}$ is the imaginary unit. Furthermore, such a $\Phi$ is strictly positive definite if the support $\supp\mu$ of measure $\mu$ has a positive Lebesgue measure. As a result, commonly used examples of strictly positive definite kernels include the Gaussian kernel $e^{-\rho\|\bz-\bz'\|^2}$, the exponential kernel $e^{-\rho\|\bz-\bz'\|}$, the Laplacian kernel $e^{-\rho\|\bz-\bz'\|_1}$ with $\rho>0$, the inverse multiquadric kernel $(c^2+\|\bz-\bz'\|^2)^{-\beta}$ with $c>0$ and $\beta>0$, and the compactly supported radial basis functions (see, e.g., Chapters 6 and 9 in \cite{Wendland2005}). However, the polynomial kernel is not strictly positive definite.

     \section{Optimality of \texorpdfstring{$\ell_0$}{}-KSVM}

   In this section, based on the characterization for the limiting subdifferential of $\ell_0$-norm hinge loss, we are able to prove the equivalent relationship among the
   the proximal stationary point, Karush-Kuhn-Tucker point, and the local minimizer.
   This is central to our theoretical analysis and algorithmic design.

    Before moving on, we present some notations. Set $\bR^m_+:=\{\bz\in\bR^m: z_i\ge 0\mbox{ for any }i\in\bN_m\}$ and $\bR^m_{-}:=\{\bz\in\bR^m: z_i\le 0\mbox{ for any }i\in\bN_m\}$.
    For every index set $I\subseteq \bN_m$, we denote by $|I|$ the number of elements in $I$ and $\overline{I}\!:=\bN_m\backslash I$ the complement of $I$.  For a vector $\bz\in \bR^{m}$, $\bz_{I}\in \bR^{|I|}$ is the vector consisting of the entries $z_i$ for $i\in I$.
     For a given $\bz\in\mathbb{R}^l$, $U_{\delta}(\bz)$ denotes the open ball
     of radius $\delta$ centered at $\bz$ on the norm $\|\cdot\|$.
     For a closed set $S\subseteq \mathbb{R}^l$ and a point $\bz\in \mathbb{R}^l$,
     ${\rm dist}(\bz,S):=\min\limits_{\bx\in S}\|\bz-\bx\|$ means the distance of $\bz$ from the set $S$.
     For a differentiable mapping $g\!:\mathbb{R}^l\to\mathbb{R}^m$,
     $\nabla\!g(\bz)$ denotes the transpose of Jacobian of $g$ at $\bz$.

    To characterize the limiting subdifferential of the $\ell_0$-norm hinge loss function,
    we recall some well-known notations of variational analysis and generalized differentiation utilized throughout the paper. The reader is referred to \cite{1998Variational,mordukhovich1994generalized,Liu2019,Liu2020} for more details.

    \begin{definition}
    (see \cite[Definition 1.5 $\&$ Lemma 1.7]{1998Variational}) Suppose that  a function $f\!:\mathbb{R}^l\to [-\infty,+\infty]$ and $\overline{\bz}\in \mathbb{R}^l$.
    The lower limit of $f$ at $\overline{\bz}$ is defined by
    \begin{align*}
    &\liminf\limits_{\bz\to \overline{\bz}}f(\bz):=\lim\limits_{\delta\downarrow 0}
    (\inf\limits_{\bz\in U_{\delta}(\overline{\bz})} f(\bz))\\
    &\quad=\min\{\alpha\in [-\infty, \infty]\,:\, \exists\, \bz^k\to \overline{\bz} \mbox{ with } f(\bz^k)\to \alpha\}.
    \end{align*}
    The function $f$ is lower semcontinuous(lsc) at $\overline{\bz}$ if
    \[
    \liminf\limits_{\bz\to \overline{\bz}}f(\bz)\geq f(\overline{\bz}),
    {\,\rm or\, equivalently\,} \liminf\limits_{\bz\to \overline{\bz}}f(\bz)=f(\overline{\bz}).
    \]
     The upper limit of $f$ at $\overline{\bz}$ is defined by
    \begin{align*}
    &\liminf\limits_{\bz\to \overline{\bz}}f(\bz):=\lim\limits_{\delta\downarrow 0}
    (\sup\limits_{\bz\in U_{\delta}(\overline{\bz})} f(\bz))\\
    &\quad=\max\{\alpha\in [-\infty, \infty]\,:\, \exists\, \bz^k\to \overline{\bz} \mbox{ with } f(\bz^k)\to \alpha\}.
    \end{align*}
    \end{definition}

    \vspace{-0.2cm}

%------------------------------------------------
   \begin{definition}\label{Gsubdiff-def}
    Consider a function $f\!:\mathbb{R}^l\to(-\infty,+\infty]$ and a point
    $\overline{\bz}\in{\rm dom}f\!:=\!\{\bz\in\mathbb{R}^l\!:\!f(\bz)<\infty\}$.
    The regular subdifferential of $f$ at $\overline{\bz}$
   is defined as
   \begin{align*}
    \widehat{\partial}f(\overline{\bz})&:=\Big\{\bv\in\mathbb{R}^l:
    \liminf_{\overline{\bz}\ne \bz\to \overline{\bz}}
    \frac{f(\bz)-f(\overline{\bz})-\langle \bv,\bz-\overline{\bz}\rangle}{\|\bz-\overline{\bz}\|}\ge 0\Big\}.
   \end{align*}
   The basic (known as the limiting) subdifferential of $f$ at $\overline{\bz}$ is defined as
   \[
    \partial f(\overline{\bz})\!:=\!\Big\{\bv\in\mathbb{R}^l: \exists\,\bz^k\!\xrightarrow[f]{}\overline{\bz}\ {\rm and}\
    \bv^k\!\in\!\widehat{\partial}f(\bz^k)\ {\rm with}\ \bv^k\!\to\! \bv\Big\},
   \]
   \end{definition}
   where  $\bz^k\xrightarrow[f]{}\overline{\bz}$ means that
   \(
     \bz^k\to \overline{\bz} {\;\rm with\;} f(\bz^k)\to f(\overline{\bz}).
   \)

     %----------------------------------------------------------------------------------
  \begin{definition}\label{NormDef}
    Let $S\subseteq\mathbb{R}^l$ be a given set.
   Consider an arbitrary $\overline{\bz}\in S$.
   The regular/Fr\'{e}chet normal cone to $S$ at $\overline{\bz}$ is defined by
   \[
     \widehat{\mathcal{N}}_{S}(\overline{\bz})
     :=\big\{\bv\in\mathbb{R}^l\,:\, \limsup_{\bz'\xrightarrow[S]{}\overline{\bz}}\frac{\langle \bv,\bz'-\overline{\bz}\rangle}{\|\bz'-\overline{\bz}\|}\leq 0\big\}
    \]
   and the limiting/Mordukhovich normal cone to $S$ at $\overline{\bz}$ is defined as
   $$
   \begin{array}{ll}
     &\displaystyle{\mathcal{N}_{S}(\overline{\bz})}
     \displaystyle{:=\limsup_{\bz\xrightarrow[S]{}\overline{\bz}}\widehat{\mathcal{N}}_{S}(\bz)  }\\
     &\quad\displaystyle{=\Big\{\bv\in \mathbb{R}^l\,:\, \exists\, \bz^k\xrightarrow[S]{}\overline{\bz},\bv^k\to \bv \mbox{ with } \bv^k\in \widehat{\mathcal{N}}_{S}(\bz^k)\Big\},}
    \end{array}
    $$
  \end{definition}
   where $ \bz^k \xrightarrow[S]{}\overline{\bz}$ means that
    \(
   \bz^k\to \overline{\bz} {\;\rm with \;} \bz^k\in S.
    \)

   When $f=\delta_{S}$ is the indicator function for a set $S\!\subseteq\mathbb{R}^l$, i.e., $\delta_S(\bz)\!=\!0$ if $\bz\!\in \!S$, otherwise $\delta_S(\bz)\!=\!+\infty$,
   the subdifferentials
    $\widehat{\partial}f(\bz)$ and $\partial f(\bz)$    reduce to  the regular
    normal cone $\widehat{\mathcal{N}}_{S}(\bz)$ and the limiting normal cone $\mathcal{N}_{S}(\bz)$, respectively.

    \begin{definition}
    Let $\mathcal{G}\!:\mathbb{R}^{l}\rightrightarrows\mathbb{R}^n$ be a given multifunction.
    the multifunction $\mathcal{G}$ is called metrically subregular
     at $\overline{\bz}$ for $\overline{\bp}\in\mathcal{G}(\overline{\bz})$ if there exists
    a constant $\kappa\ge 0$ along with $\delta>0$ such that
   \begin{equation*}\label{subregular}
    {\rm dist}(\bz,\mathcal{G}^{-1}(\overline{\bp}))\!\leq\! \kappa{\rm dist}(\overline{\bp},\mathcal{G}(\bz))
   \,{\rm for\, any}\, \bz\!\in \!U_{\delta}(\overline{\bz}).
  \end{equation*}
    \end{definition}
   Metrical subregularity was introduced
   by Ioffe in \cite{ioffe1979regular} (under a different name) as a constraint qualification
   related to equality constraints in nonsmooth optimization problems, and was later
   extended in \cite{dontchev2004regularity} to generalized equations.
   The metrical subregularity of a multifunction has already been studied by many authors
  under various names (see, e.g., \cite{henrion2005calmness,ioffe2008metric,gfrerer2011first,bai2019directional}
  and the references therein).

   \subsection{Limiting subdifferential of the \texorpdfstring{$\ell_0$}{}-norm hinge loss}

  %--------------------------------------------------------------
    In this subsection, we will drive the explicit expression of  the limiting subdifferential of
    the $\ell_0$-norm hinge loss function. To this end, we consider
    the following general composite mapping  $\varphi: \mathbb{R}^l\to \mathbb{R}$ by
    \begin{equation}\label{varphiDef}
    \varphi(\bz):=CL_{0/1}(g(\bz)), \mbox{ for any }\bz\in \mathbb{R}^{l}
    \end{equation}
    where $g:\mathbb{R}^{l}\to \mathbb{R}^{m}$ is a differentiable map.

   \begin{lemma}\label{subdiff-Cgz}
   Suppose that  $\varphi: \mathbb{R}^l\to \mathbb{R}$ is defined as in \eqref{varphiDef}.
   Fix any $\overline{\bz}\in\!\mathbb{R}^l$. Write
   ${\rm ps}(g(\overline{\bz})):=\{i\in \bN_m\,:\, g_i(\overline{\bz})>0\}$ and $I\!:={\rm ps}(g(\overline{\bz}))$
   and $\Theta_{\overline{I}}\!:=\!\big\{\bz\in\!\mathbb{R}^l\,:\,g_{\overline{I}}(\bz)\leq 0\big\}$
   with $g_{\overline{I}}(\bz):=(g(\bz))_{\overline{I}}$.
    Then,
  \begin{itemize}
    \item [(i)] $\widehat{\partial}\varphi(\overline{\bz})
           =\widehat{\mathcal{N}}_{\Theta_{\overline{I}}}(\overline{\bz})$
                and $\partial\varphi(\overline{\bz})=\mathcal{N}_{\Theta_{\overline{I}}}(\overline{\bz})$.

   \item[(ii)] When the mapping $\mathcal{G}(\bz)\!:=\!g_{\overline{I}}(\bz)-\mathbb{R}_{-}^{|\overline{I}|}$
                is subregular at $\overline{\bz}$ for the origin,
                \begin{align}\label{L01SubEq}
                 \widehat{\partial}\varphi(\overline{\bz})
                  =\partial\varphi(\overline{\bz})
                =\nabla\!g_{\overline{I}}(\overline{\bz})
                 \mathcal{N}_{\mathbb{R}_{-}^{|\overline{I}|}}(g_{\overline{I}}(\overline{\bz})).
               \end{align}
   \end{itemize}
   \end{lemma}

   The Lemma \ref{subdiff-Cgz} is motivated by Lemma 2.2 in \cite{Liu2022} or Proposition 3.2 in \cite{Wu2021}.
   For completeness,  its proof is given in Appendix \ref{AppendixA}.

   \begin{remark}\label{remark-Varphi}
    When  $g$ is an affine mapping given by
   \begin{equation*}
    g(\bz)=B\bz\!-\bd\ \ {\rm for}\ B\in\mathbb{R}^{m\times l}\ {\rm and}\ \bd\in\mathbb{R}^m,
    \end{equation*}
    we know from \cite[Page 211]{ioffe2008metric} and \cite[Corollary 3]{1999Strong} that
    the mapping $\mathcal{G}$ in Lemma \ref{subdiff-Cgz} (ii) is metrically subregular at
    $\overline{\bz}\in \bR^{l}$ for the origin.
   \end{remark}

   In Lemma \ref{subdiff-Cgz} and Remark \ref{remark-Varphi}, taking $g(\bz)\!=\bz$ for any $\bz\!\in\! \mathbb{R}^{l}$,
  we have the explicit expression of  the limiting subdifferential of
    $\ell_0$-norm hinge loss function.

  \begin{theorem}\label{SubDifL01Th1}
  For any $\bu\!\in \!\bR^m$, write $\overline{I}_0\!=\!\{i\!\in\! \bN_m : u_i=0\}$. Then, we have
  \begin{equation*}
  \partial L_{0/1}(\bu)\!=\!\{\lambda\!\in\! \mathbb{R}^m: \lambda_i\!\in\!\bR_+, i\!\in\! \overline{I}_0,\,{\rm otherwise}\, \lambda_i=0 \}.
  \end{equation*}
  In particular, we obtain the limiting subdifferential of $\ell_0$-norm hinge loss function with
  for any $t\in \bR$
  $$
   \partial \ell_{0/1}(t)\left\{\begin{array}{ll}
   =0,& \mbox{ if } t\ne0,\\
   \in \bR_+,&\mbox{ if } t=0.\\
   \end{array}\right.
  $$
  \end{theorem}

  In Lemma \ref{subdiff-Cgz} and Remark \ref{remark-Varphi},
  taking
  \[
  g(\bc,b)\!=\!{\bf 1}-\diag(\by)\bK\bc-b\by
  \]
 for any $(\bc,b)\!\in\! \mathbb{R}^{m}\times \bR$  we have the following conclusion.

  \begin{theorem}\label{SubDifL01Th2}
   For any $(\bc,b)\!\in\!\!\mathbb{R}^m\times \bR$. Write
   $I\!:=\!\{i\in \bN_m: g_i(\bc, b)>0\}$,
    $\overline{I}_0\!:=\!\{i\in \bN_m: g_i(\bc, b)=0\}$ and
     $\overline{I}_{<}\!:=\!\{i\in \bN_m: g_i(\bc, b)<0\}$.
   Then, $\bv\in \partial \varphi(\bc,b)$ if and only if there exist $\boldsymbol{\lambda}\in \bR^m$
    such that
   \begin{align*}
       \bv=\begin{pmatrix}\bK{\rm diag}(\by)\\\by^{\top}\end{pmatrix}\boldsymbol{\lambda}
  \mbox{ with }
  \lambda_i\in \left\{\begin{array}{cl}
        \bR_- ,   &  {\rm if\;} i\in \overline{I}_0,\\
        0 ,   &  {\rm if\;} i\in I\cup\overline{I}_{<}.
       \end{array}\right.
   \end{align*}
   \end{theorem}
   \begin{remark}
   Theorem \ref{SubDifL01Th2} plays a key importance to the proof that any  local optimal solution to the problem \eqref{SVM01eq2} is also a proximal stationary point.
   In fact, according to \eqref{varphiDef}, the function $\varphi$ is a composition function  of the $\ell_0$-norm hinge loss function and a linear mapping.
   Based on the characterization of the limiting subdifferential of $\varphi$, Theorem \ref{OptiTh1} later will prove that any local optimal solution to the problem \eqref{SVM01eq2} must be a proximal stationary point without any condition. However, \cite[Theorem 3.2]{Wang2022}) only proved any global optimal solution  is also a proximal stationary point under some very strong conditions. In addition, by  the limiting subdifferential of $\varphi$,  we will propose the definition of Karush-Kuhn-Tucker (KKT) point to the problem \eqref{SVM01eq2} and then prove the equivalent
 relationship among the proximal stationary point, the KKT point, and the local optimal solution, see Theorems \ref{KKTTheorm} and  \ref{OptiTh1} later.
   \end{remark}

  \subsection{First-order optimality conditions}

   In this subsection, we will discuss the first-order optimality conditions for the problem \eqref{SVM01eq2}. To this end, we shall present the definitions of the KKT point and the proximal stationary point.

   \begin{definition}\label{KKTdef}
   Consider the problem \eqref{SVM01eq2}, the point $(\bc^*,b^*,\bu^*)$ is called a KKT point,
   if there exists a Lagrangian multiplier
    $\boldsymbol{\lambda}^*\in \bR^m$ such that
   \begin{equation}\label{KKT}
   \left\{\begin{array}{ll}
   &\bK\bc^{*}+\bK\diag(\by)\boldsymbol{\lambda}^{*}={\bf 0},\\
   &\langle \by,\boldsymbol{\lambda}^{*}\rangle=0,\\
   & \bu^{*}+\diag(\by)\bK\bc^{*}+b^{*}\by={\bf 1},\\
  & {\bf 0}\in C\partial \|\bu_+^{*}\|_0+\boldsymbol{\lambda}^{*}.
  \end{array}\right.
 \end{equation}
   \end{definition}

    \begin{definition}[$L_{0/1}$ Proximal Operator]\label{L01Proximity} For any given $\gamma,C>0$ and $\boldsymbol{\eta}\in\bR^m$, the proximal operator of $L_{0/1}$ given by \eqref{L01norm}, is defined by
     \begin{equation}\label{Prox0}
    {\rm Prox}_{\gamma C L_{0/1}}({\boldsymbol\eta})\!:=\!\arg\min_{\bv\in\bR^m} CL_{0/1}(\bv)\!+\!\frac{1}{2\gamma}\|\bv\!-\!{\boldsymbol\eta}\|^2.
    \end{equation}
    We call such an operator the $L_{0/1}$ proximal operator.
    \end{definition}
   It was proved in \cite{Wang2022} that the $L_{0/1}$ proximal operator defined by \eqref{Prox0} admits a closed form solution:
    $$
    [{\rm Prox}_{\gamma C L_{0/1}}({\boldsymbol\eta})]_i
    :=\left\{
    \begin{array}{ll}
      \{0\},   &  {\rm if\;} 0<\eta_i< \sqrt{2\gamma C}, \\
      \{0,\eta_i\},   &{\rm if\;} \eta_i=\sqrt{2\gamma C},\\
      \{\eta_i\},   & {\rm otherwise}.
    \end{array}
    \right.
   $$
    To guarantee the injectivity of the $L_{0/1}$ proximal operator, we adopt the modified version as follows:
   \begin{align}\label{Prox}
    [{\rm Prox}_{\gamma C L_{0/1}}({\boldsymbol\eta})]_i
    :=\left\{
    \begin{array}{cl}
      \{0\},   &  {\rm if\;} 0<\eta_i\leq \sqrt{2\gamma C} \\
      \{\eta_i\},   & {\rm otherwise}
    \end{array}
    \right., i\in\bN_m.
    \end{align}

    \begin{definition}\label{PstationPointDef}
    For a given $C>0$, we say $(\bc^*,b^*,\bu^*)$ is a proximal stationary point of  $\ell_0$-KSVM \eqref{SVM01eq2} if there exists a Lagrangian multiplier
    $\boldsymbol{\lambda}^*\in \bR^m$ and a constant $\gamma>0$ such that

   \begin{subnumcases}{}
   \bK\bc^{*}+\bK{\rm diag}(\by)\boldsymbol{\lambda}^{*} &=   ${\mathbf 0}$,\label{Pseq1}\\
   \langle \by,\boldsymbol{\lambda}^*\rangle & = ${\mathbf 0}$,\label{Pseq2}\\
   \bu^*+{\rm diag}(\by)\bK \bc^*+b^*\by & =  ${\bf 1}$,\label{Pseq3}\\
   {\rm Prox}_{\gamma C L_{0/1}}(\bu^*-\gamma     \boldsymbol{\lambda}^*) & =$\bu^*$,\label{Pseq4}
   \end{subnumcases}
    where ${\rm Prox}_{\gamma C L_{0/1}}(\cdot)$ is given by \eqref{Prox}.

    \end{definition}

    The following theorem reveals the equivalent relationship between the KKT point and a proximal stationary point of $\ell_0$-KSVM \eqref{SVM01eq2}.
    \begin{theorem}\label{KKTTheorm}
    Consider the problem \eqref{SVM01eq2}. The point $(\bc^*,b^*,\bu^*)$ with the  Lagrangian multiplier $\boldsymbol{\lambda}^*\in\bR^m$ is a proximal stationary  point if and only if  $(\bc^*,b^*,\bu^*)$ with  $\boldsymbol{\lambda}^*$ is also a KKT point.
    \end{theorem}
    \begin{proof}
    Suppose that $(\bc^*,b^*,\bu^*)$ with the  Lagrangian multiplier $\boldsymbol{\lambda}^*\in\bR^m$ is a proximal stationary  point of  \eqref{SVM01eq2}, then $(\bc^*,b^*,\bu^*,\boldsymbol{\lambda}^*)$ satisfies
    equations \eqref{Pseq1}-\eqref{Pseq4}. Obviously, by the definition of KKT, it is sufficient to argue that ${\bf 0}\in C\partial \|\bu_+^{*}\|_0+\boldsymbol{\lambda}^{*}$.
    Notice ${\rm Prox}_{\gamma C L_{0/1}}(\bu^*-\gamma \boldsymbol{\lambda}^*)=\bu^*$ and
    \eqref{Prox0}, we have that
     \begin{equation*}
     \bu^*=\arg\min_{\bv\in\bR^m} CL_{0/1}(\bv)+\frac{1}{2\gamma}\|\bv-(\bu^*-\gamma \boldsymbol{\lambda}^*)\|^2.
    \end{equation*}
    From optimal condition \cite[Theorem 10.1]{1998Variational}, it follows that
    \[
    {\bf 0}\in \partial ( CL_{0/1}(\bu^*))+\frac{1}{\gamma}(\bu^*-(\bu^*-\gamma \boldsymbol{\lambda}^*)),
    \]
    which is equivalent with  ${\bf 0}\!\in\! C\partial \|\bu_+^{*}\|_0+\boldsymbol{\lambda}^{*}$.
    The desired result is obtained.

    \medskip

    Conversely, assume that $(\bc^*,b^*,\bu^*)$ with the  Lagrangian multiplier $\boldsymbol{\lambda}^*\in\bR^m$ is a KKT point. Then the equations in \eqref{KKT} holds. Obviously, by the definition of the proximal stationary point, it is sufficient to argue that the equation \eqref{Pseq4} holds.
    Denote $I:=\{i\in \bN_m : u^*_i>0\}$, $\overline{I}=\bN_m\setminus I$ and  $\overline{I}_0:=\{i\in \overline{I} : u^*_i=0\}$.
    Notice ${\bf 0}\in C\partial \|\bu_+^{*}\|_0+\boldsymbol{\lambda}^{*}$.
    From Theorem \ref{SubDifL01Th1}, it holds that
    $$
    \lambda^*_i\left\{\begin{array}{ll}
    \le 0,  &\mbox{ if }i\in  \overline{I}_{0},\\
    0,  &\mbox{ otherwise}.\\
    \end{array}\right.
    $$
    Write $\overline{I}_{0-}:=\{i\in \overline{I}_0 : \lambda^*_i<0\}$.
    Take
    $$
   \gamma\;\left\{\begin{array}{ll}
   =\min\big\{\frac{2C}{\max_{i\in \overline{I}_0}(\lambda^*_i)^2}, \frac{\min_{i\in I}(u^*_i)^2}{2C}\big\},&
   \mbox{ if }\overline{I}_{0-}\ne\emptyset, I\ne\emptyset,\\                                                                           =\frac{\min_{i\in I}(u^*_i)^2}{2C}, &\mbox{ if }\overline{I}_{0-}=\emptyset, I\ne\emptyset, \\
   =\frac{2C}{\max_{i\in \overline{I}_0}(\lambda^*_i)^2}, &\mbox{ if }\overline{I}_{0-}\ne\emptyset,I=\emptyset, \\
   \in (0,+\infty),& \mbox{ otherwise}. \\
   \end{array}\right.
   $$
   Since
   \begin{align*}
    (\bu^*-\gamma \boldsymbol{\lambda}^*)_i
    =\left\{
    \begin{array}{cl}
      -\gamma \lambda^*_i,   &  {\rm if\;} i\in \overline{I}_0, \\
      u^*_i,   & {\rm otherwise},
    \end{array}
    \right.
    \end{align*}
    and \eqref{Prox0}, we know that ${\rm Prox}_{\gamma C L_{0/1}}(\bu^*-\gamma \boldsymbol{\lambda}^*) =\bu^*$. Hence, $(\bc^*,b^*,\bu^*)$ is a proximal stationary point.
    \end{proof}

    The following theorem reveals the equivalent relationship between the local minimizer and a proximal stationary point of $\ell_0$-KSVM \eqref{SVM01eq2}.  Its proof is shown in Appendix \ref{AppendixB}.
    \begin{theorem}\label{OptiTh1}
    Consider the problem \eqref{SVM01eq2}. Then $(\bc^*,b^*,\bu^*)$ is  a proximal stationary point if and only if it is also a locally optimal solution.
    \end{theorem}

    \section{ADMM algorithm}

     In this section, we shall develop the ADMM algorithm for $\ell_0$-KSVM \eqref{SVM01eq2}.
     To start with,  we discuss support vectors which are characterized by the $L_{0/1}$ proximal operator. Then, we build the ADMM algorithm for \eqref{SVM01eq2}. After that, we provide the convergence analysis for the proposed ADMM algorithm.

    \subsection{Support vectors}

    Let $(\bc^*,b,\bu^*)$ be a proximal stationary point of $\ell_0$-KSVM \eqref{SVM01eq2}.  Then from Definition \eqref{PstationPointDef}, there is a Lagrangian multiplier
    $\boldsymbol{\lambda}^*\!\in\! \bR^{m}$ and a constant $\gamma\!>\!0$ such that \eqref{Pseq1}-\eqref{Pseq4} hold. Denote an index set
   \begin{align}\label{SVdefGamma}
   \Gamma_*&:=\{i\in \bN_m:\,
   u^*_i-\gamma\lambda^*_i\in (0,\sqrt{2\gamma C}\,]\},
    \end{align}
   and
    $\overline{\Gamma}_*\!=\!\bN_m\!\setminus\! \Gamma_*$.
   Combining \eqref{Pseq4} and \eqref{Prox}, we know that
   \begin{align*}
       \bu^*&\overset{\eqref{Pseq4}}{=}{\rm Prox}_{\gamma C L_{0/1}}(\bu^*-\gamma     \boldsymbol{\lambda}^*)\\
       &=\begin{bmatrix}
       ({\rm Prox}_{\gamma C L_{0/1}}(\bu^*-\gamma     \boldsymbol{\lambda}^*))_{\Gamma_*}\\
       ({\rm Prox}_{\gamma C L_{0/1}}(\bu^*-\gamma     \boldsymbol{\lambda}^*))_{\overline{\Gamma}_*}
       \end{bmatrix}\\
       &\overset{\eqref{Prox}}{=}\begin{bmatrix}
       {\bf 0}_{\Gamma_*}\\
       (\bu^*-\gamma\boldsymbol{\lambda}^*)_{\overline{\Gamma}_*}
       \end{bmatrix}
   \end{align*}
   which is equivalent to
   \begin{align}\label{SVTemp1}
      \begin{bmatrix}
         \bu^*_{{\Gamma}_*}\\
         \boldsymbol{\lambda}^*_{\overline{\Gamma}_*}
      \end{bmatrix}={\bf 0}.
   \end{align}
   Then $\Gamma_*$ in \eqref{SVdefGamma} is equivalent to
   \begin{align*}
   \Gamma_*&=\{i\in \bN_m:\,
    \lambda^*_i\in [-\sqrt{2C/\gamma},0)\}.
    \end{align*}
    This and \eqref{SVTemp1} result in
   \begin{align}\label{SVTemp2}
       \lambda^*_i\left\{\begin{array}{ll}
       \in [-\sqrt{2C/\gamma},0), & {\rm if\; } i\in \Gamma_*,\\
       =0, & {\rm if\;} i\in \overline{\Gamma}_*.
       \end{array}
       \right.
   \end{align}
   Taking  \eqref{SVTemp2} and \eqref{WReTh} into \eqref{Pseq1} derives
   \begin{align*}
   \begin{array}{ll}
      & \begin{pmatrix}
          \langle \phi(\bx_1),\bw^*\rangle_{\cF}\\
          \langle \phi(\bx_2),\bw^*\rangle_{\cF}\\
          \vdots\\
          \langle \phi(\bx_m),\bw^*\rangle_{\cF}
       \end{pmatrix}
       =\bK \bc^*
       =-\bK{\rm diag}(\by)\boldsymbol{\lambda^*} \\
       &\qquad\qquad=-\begin{pmatrix}
          \langle \phi(\bx_1),\sum\limits_{i\in \Gamma_*}y_i\lambda^*_i\phi(\bx_i)\rangle_{\cF}\\
          \langle \phi(\bx_2),\sum\limits_{i\in \Gamma_*}y_i\lambda^*_i\phi(\bx_i)\rangle_{\cF}\\
          \vdots\\
          \langle \phi(\bx_m),\sum\limits_{i\in \Gamma_*}y_i\lambda^*_i\phi(\bx_i)\rangle_{\cF}
       \end{pmatrix}
   \end{array}
   \end{align*}
   which implies that if $[\phi(\bx_1),\phi(\bx_2),\cdots,\phi(\bx_m)]$ has full column rank (or equivalently $K$ is strictly positive definite), then it holds that
   \begin{equation}\label{wsv}
   \bw^*=-\sum\limits_{i=1}^m y_i\lambda^*_i\phi(\bx_i)=\sum\limits_{i\in \Gamma_*}y_i(-\lambda^*_i)\phi(\bx_i).
   \end{equation}

   \begin{remark}\label{remarksv}
     \begin{itemize}
      \item [\bf (i) ] By Equation \ref{wsv}, we call that any $\bx_i, i\in \Gamma_*$ is a support vector.

      \item [\bf(ii)] By \eqref{SVTemp1},
      \eqref{Pseq3} implies that
      \[
      {\bf 1}\!=\!\bu^*_{\Gamma_*}\!+({\rm diag}(\by)\bK \bc^*+b^*\by)_{\Gamma_*}
      \!=\!({\rm diag}(\by)\bK \bc^*+b^*\by)_{\Gamma_*}
      \]
      which and \eqref{WReTh} yield
      \[
      \langle \bw^*,\phi(\bx_i)\rangle_{\cF}+b^*=\pm{1}, \mbox{ for any }i\in \Gamma_*.
      \]
      In other words, any support vector $\bx_i$ of $\ell_0$-KSVM must fall into the canonical hyperplanes $\langle \bw^*,\phi(\bx_i)\rangle_{\cF}+b^*=\pm{1}$.
       \end{itemize}
     \end{remark}

    This phenomenon shows that the $\ell_0$-KSVM could render fewer support vectors than the other soft-margin KSVM models, which is also verified by our numerical experiments. For example, see Figures \ref{Fig1} a)-c) for Double Circles in Section 5.

    \subsection{ADMM for \texorpdfstring{$\ell_0$}{}-KSVM}

    We shall apply ADMM \cite{boyd2011distributed} to $\ell_0$-KSVM.
    The augmented Lagrangian function of the problem \eqref{SVM01eq2} is given by
     \begin{equation}\label{Lagrangian1}
     \begin{array}{ll}
     L_{\sigma}(\bc,b,\bu;\boldsymbol{\lambda}):
     &=\frac{1}{2}\bc^{\top}\bK\bc+C\|\bu_+\|_0  \\
     &+\langle \boldsymbol{\lambda},  \bu  +{\rm diag}(\by)\bK \bc+b\by -{\bf 1}\rangle\\
     &+\frac{\sigma}{2}\| \bu+{\rm diag}(\by)\bK \bc+b\by-{\bf 1}\|^2
     \end{array}
     \end{equation}
     where $\boldsymbol{\lambda}$ is the Lagrangian multiplier and
     $\sigma\!>\!0$ is the penalty parameter.
     The above function \eqref{Lagrangian1} is reformulated as

    \begin{equation}
    \begin{array}{ll}\label{Lagrangian2}
    &L_{\sigma}(\bc,b,\bu;\boldsymbol{\lambda})
    =\frac{1}{2}\bc^{\top}\bK\bc+C\|\bu_+\|_0 \\
    &\qquad+\frac{\sigma}{2}\|\bu -\big({\bf 1}-{\rm diag}(\by)\bK \bc-b\by-\frac{\boldsymbol{\lambda}}{\sigma}\big)\|^2-\frac{\|\boldsymbol{\lambda}\|^2}{2\sigma}.
    \end{array}
    \end{equation}
    Therefore, the iterative formulas at the $k$ iteration are given as follows:
    \begin{subnumcases}{}
    \bu^{k+1}&:=$\arg\min_{\bu\in\bR^m}L_{\sigma}(\bc^k,b^k,\bu;\boldsymbol{\lambda}^k)$ \label{ADMMa}\\
    \bc^{k+1}&:=$\arg\min_{\bc\in\bR^m}L_{\sigma}(\bc,b^k,\bu^{k+1};\boldsymbol{\lambda}^k)$   \label{ADMMb}\\
     b^{k+1}&:=$\arg\min_{b\in\bR}L_{\sigma}(\bc^{k+1},b,\bu^{k+1};\boldsymbol{\lambda}^k)$   \label{ADMMc}\\
    \boldsymbol{\lambda}^{k+1}&:={\scriptsize$\boldsymbol{\lambda}^k\!+\!\iota\sigma(\bu^{k+1}\!+\!{\rm diag}(\by)\bK \bc^{k+1}+b^{k+1}\by \!-\!{\bf 1})$}   \label{ADMMd}
    \end{subnumcases}
   where $\iota>0$ is the dual step-size.

   To be precise, we should calculate each sub-problem in \eqref{ADMMa}-\eqref{ADMMd} as follows.
   Denote
   \[
   \boldsymbol{\eta^{k}}:={\bf 1}-{\rm diag}(\by)\bK\bc^{k}-b^k\by-\frac{\boldsymbol{\lambda}^k}{\sigma}
   \]
    and define index sets $\Gamma_k$ at the $k$th step by
    \begin{align}\label{SVdefGammaK}
   \Gamma_k&=\Big\{i\in \bN_m:\,
    \eta^k_i\in (0,\sqrt{\frac{2C}{\sigma}}\;]\Big\}
    \end{align}
    and $\overline{\Gamma}_k=\bN_m\!\setminus\! \Gamma_k$.

   (i) Updating $\bu^{k+1}$: By \eqref{Prox0}, \eqref{Prox} and \eqref{Lagrangian2}, we have
   \begin{equation}\label{ADMMeq1}
   \bu^{k+1}
   =\text{Prox}_{\frac{C}{\sigma}\|(\cdot)_+\|_0}(\boldsymbol{\eta}^k).
   \end{equation}
   Together with \eqref{Prox} and \eqref{SVdefGammaK}, we conclude that
    \begin{equation}\label{ADMMuk}
   \bu^{k+1}_{\Gamma_k}={\bf 0},\qquad \bu^{k+1}_{\overline{\Gamma}_k}=\boldsymbol{\eta^k}_{\overline{\Gamma}_k}.
   \end{equation}

    (ii) Updating $\bc^{k+1}$: By \eqref{Lagrangian2}, we see that $L_{\sigma}(\bc,b^k,\bu^{k+1};\boldsymbol{\lambda}^k)$ is a convex quadratic function. Denote  $\boldsymbol{\xi^{k}}:={\bf 1}-\bu^{k+1}-b^k\by-\frac{\boldsymbol{\lambda}^k}{\sigma}$,  we compute
    $$
    \begin{array}{ll}
    &\displaystyle{\nabla_{\bc} L_{\sigma}(\bc,b^k,\bu^{k+1};\boldsymbol{\lambda}^k)}\\
    &\quad\displaystyle{=\nabla_{\bc}\Big(\frac{1}{2}\bc^{\top}\bK\bc+\frac{\sigma}{2}\|\diag(\by)\bK \bc-\boldsymbol{\xi}^k\|^2\Big)  }\\
    &\quad\displaystyle{=\big[\bK+\sigma \bK\bK\big]\bc-\sigma\bK\diag(\by)\boldsymbol{\xi}^k,}
    \end{array}
    $$
    where we have used two facts that $\diag(\by)\diag(\by)={\mathbb I}_m$ with ${\mathbb I}_m$ being an $m\times m$ identity matrix, and $\bK^{\top}=\bK$ in the second equality.
    Then $\bc^{k+1}$ can be obtained by solving the following linear equation:
    \begin{equation}\label{ADMMeq2}
    \big[\bK+\sigma \bK\bK\big]\bc^{k+1}=\sigma\bK\diag(\by)\boldsymbol{\xi}^k.
   \end{equation}
   In particular, if $K$ is further a strictly positive definite kernel, then \eqref{ADMMeq2} reduces to
   $$
   \big[\frac{1}{\sigma} {\mathbb I}_m+ \bK\big] \bc^{k+1}=\diag(\by)\boldsymbol{\xi}^k.
   $$

  (iii) Updating $b^{k+1}$: Observe that
  $$
  \begin{array}{ll}
  b^{k+1}
  &\displaystyle{\!=\!\arg\min_{b\in\bR}\|\bu^{k+1}\!-\!\big({\bf 1}\!-\!{\rm diag}(\by)\bK \bc^{k+1}\!-\! b\by\!-\!\frac{\boldsymbol{\lambda}^{k+1}}{\sigma}\big)\|^2 }\\
  &\displaystyle{\!=\!\arg\min_{b\in\bR}\|b\by\!-\!\big({\bf 1}\!-\!\bu^{k+1}\!-\!{\rm diag}(\by)\bK \bc^{k+1}\!-\!\frac{\boldsymbol{\lambda}^{k+1}}{\sigma}\big)\|^2.}
  \end{array}
  $$
  It follows that
   \begin{equation}\label{ADMMeq3}
   b^{k+1}=\frac{\by^{\top}\br^k}{\by^{\top}\by}=\frac{\by^{\top}\br^k}{m},
   \end{equation}
  where $\br^k:={\bf 1}-\bu^{k+1}-{\rm diag}(\by)\bK \bc^{k+1}-\frac{\boldsymbol{\lambda}^{k+1}}{\sigma}$.

   (iv) Updating $\boldsymbol{\lambda}^{k+1}$: We update $\boldsymbol{\lambda}^{k+1}$
   in  \eqref{ADMMd} as follows
   \begin{equation}\label{ADMMeq4}
     \boldsymbol{\lambda}^{k+1}_{\Gamma_k}=\boldsymbol{\lambda}^{k}_{\Gamma_k}+\iota\sigma\boldsymbol{\omega}^{k+1}_{\Gamma_k},
     \boldsymbol{\lambda}^{k+1}_{\overline{\Gamma}_k}=0,
   \end{equation}
   where $\boldsymbol{\omega}^{k+1}:=\bu^{k+1}+{\rm diag}(\by)\bK \bc^{k+1}+b^{k+1}\by -{\bf 1}$. And, we set $\boldsymbol{\lambda}^{k+1}_{\overline{\Gamma}_k}=0$ in \eqref{ADMMeq4} motivated by \eqref{SVTemp1}.

   To sum up, the ADMM algorithm for $\ell_0$-KSVM \eqref{SVM01eq2} is stated in Algorithm \ref{algo:KSVM}.
  \begin{algorithm}
  \caption{ADMM for $\ell_0$-KSVM}
  \label{algo:KSVM}
  \begin{enumerate}
  \item Initialize $(\bc^0,b^0,\bu^0,\boldsymbol{\lambda}^0)$. Set $C,\sigma,\iota>0$. Choose a positive definite kernel $K$ on $\bR^d$ and the maximum iteration $N$. If the tolerance is satisfied or $k\le N$, do
  \item Update $\bu^{k+1}$ as in \eqref{ADMMuk};
  \item  Update $\bc^{k+1}$ as in \eqref{ADMMeq2};
  \item Update $b^{k+1}$ as in \eqref{ADMMeq3};
  \item  Update $\boldsymbol{\lambda}^{k+1}$ as in \eqref{ADMMeq4};
  \item Set $k=k+1$
  \item Return the final solution $(\bc^k,b^k,\bu^k,\boldsymbol{\lambda}^k)$ to \eqref{SVM01eq2}.
  \end{enumerate}
  \end{algorithm}

  If $K$ is a strictly positive definite kernel on $\bR^d$, then by \eqref{WReTh} and \eqref{Pseq1}, we have the $\ell_0$-KSVM hyperplane
  \begin{equation*}
  h(\bx)=\langle \bw^*,\phi(\bx)\rangle_{\cF}+b^*=-\sum_{i\in\Gamma_*}y_i\lambda_i^*K(\bx_i,\bx)+b^*.
  \end{equation*}

  \subsection{Convergence Analysis}
  In this subsection, we will prove that if the sequence generated by Algorithm \ref{algo:KSVM} has a limit point, then it must be a proximal stationary
  point for $\ell_0$-KSVM \eqref{SVM01eq2}.
  \begin{theorem}\label{ConvThm}
    Let $(\bc^*,b^*,\bu^*,\boldsymbol{\lambda}^*)$ be the limit point of the sequence generated by ADMM
    for $\ell_0$-KSVM \eqref{SVM01eq2}.
    Then $(\bc^*,b^*,\bu^*)$ is a proximal stationary point with
    $\gamma\!=\!\frac{1}{\sigma}$ and also a locally optimal solution to the problem \eqref{SVM01eq2}.
  \end{theorem}
  \begin{proof} Suppose that $(\bc^k,b^k,\bu^k,\boldsymbol{\lambda}^k)$ is generated by Algorithm  \ref{algo:KSVM} and
    $(\bc^*,b^*,\bu^*,\boldsymbol{\lambda}^*)$ is its limits point.  Denote
   \[
   \boldsymbol{\eta^{k}}:={\bf 1}-{\rm diag}(\by)\bK\bc^{k}-b^k\by-\frac{\boldsymbol{\lambda}^k}{\sigma}
   \]
    and define index sets $\Gamma_k$ with respect to $\boldsymbol{\eta^{k}}$ at the $k$th step by
   $$
   \Gamma_k=\Big\{i\in \bN_m:\,\eta^k_i\in (0,\sqrt{\frac{2C}{\sigma}}\;]\Big\}.
   $$
    Notice that $\Gamma_k\subseteq \bN_m$ for all $k\in\bN$. Then there exists a subset $\Gamma\subseteq \bN_m$ and an index  set $J\subseteq \bN$ such that
   the subsequence $\{\Gamma_k\}_{k\in J}$ satisfying
    \[
    \Gamma_k\equiv\Gamma, \mbox{ for any } k\in J.
    \]
    Taking the limit along with $J$ of \eqref{ADMMeq4}, we know that
     \begin{equation}\label{CovThTemp1}
     \boldsymbol{\lambda}^{*}_{\Gamma}=\boldsymbol{\lambda}^{*}_{\Gamma}+\iota\sigma\boldsymbol{\omega^{*}}_{\Gamma},\qquad
     \boldsymbol{\lambda}^{*}_{\overline{\Gamma}}={\bf 0}.
   \end{equation}
   and
   \begin{equation}\label{CovThTemp2}
   \boldsymbol{\omega}^{*}=\bu^{*}+{\rm diag}(\by)\bK \bc^{*}+b^{*}\by -{\bf 1}.
   \end{equation}
   The first equation in \eqref{CovThTemp1} can
   yields that $\boldsymbol{\omega}^{*}_{\Gamma}={\bf 0}$.  Taking the limit along with $J$ of \eqref{ADMMuk} and $\boldsymbol{\eta}^k$
   respectively, we have that
   \begin{align}\label{CovThTemp3}
    \bu^*_{\Gamma}={\bf 0}, \quad \bu^*_{\overline{\Gamma}}=\boldsymbol{\eta}^*_{\overline{\Gamma}}
   \end{align}
   and
   \begin{align}\label{CovThTemp4}
       \boldsymbol{\eta}^*&={\bf 1 }-{\rm diag}(\by)\bK\bc^*-b^*\by-\frac{\boldsymbol{\lambda}^*}{\sigma}\nonumber\\
       &={\bf 1 }-{\rm diag}(\by)\bK\bc^*-b^*\by-\bu^*+\bu^*-\frac{\boldsymbol{\lambda}^*}{\sigma}\nonumber\\
       &\overset{\eqref{CovThTemp2}}{=}-\boldsymbol{\omega}^*+\bu^*-\frac{\boldsymbol{\lambda}^*}{\sigma}
      \end{align}
    which implies that $\boldsymbol{\omega}^*_{\overline{\Gamma}}={\bf 0}$ by the second equation in \eqref{CovThTemp1}. Together with
     $\boldsymbol{\omega}^{*}_{\Gamma}={\bf 0}$, we know that $\boldsymbol{\omega}^*={\bf 0}$ and then from \eqref{CovThTemp3}, it follows
   \begin{equation}\label{CovThTemp5}
   \bu^{*}+{\rm diag}(\by)\bK \bc^{*}+b^{*}\by ={\bf 1}.
   \end{equation}
   Since $\boldsymbol{\omega}^*={\bf 0}$, \eqref{CovThTemp4} yields $\boldsymbol{\eta}^*=\bu^*-\frac{\boldsymbol{\lambda}^*}{\sigma}$. Together this with \eqref{CovThTemp3},\eqref{SVdefGammaK} and \eqref{Prox}, we can conclude that
   \begin{equation}\label{CovThTemp6}
   \bu^{*}= {\rm Prox}_{\frac{C}{\sigma} L_{0/1}}(\boldsymbol{\eta}^*)= {\rm Prox}_{\gamma C L_{0/1}}(\bu^*-\frac{\boldsymbol{\lambda}^*}{\sigma}).
   \end{equation}
   Taking the limit along with $J$ of \eqref{ADMMeq3}, we have that
   \begin{align*}
   b^*&=\frac{\langle \by, {\bf{1}}-\bu^*-{\rm diag}(\by)\bK\bc^*-\frac{\boldsymbol{\lambda}^*}{\sigma}\rangle}{m}\\
   &\overset{\eqref{CovThTemp5}}{=}\frac{\langle \by, b^*\by-\frac{\boldsymbol{\lambda}^*}{\sigma}\rangle}{m},
   \end{align*}
   which implies that
   \begin{align}\label{CovThTemp7}
       \langle \by, \boldsymbol{\lambda}^*\rangle=0.
   \end{align}
   Taking the limit along with $J$ of \eqref{ADMMeq2}, we
   have that
    \begin{align*}
   (\bK+\sigma \bK\bK)\bc^*&=\sigma\bK{\rm diag}(\by)
   ( {\bf{1}}-\bu^*-b^*\by-\frac{\boldsymbol{\lambda}^*}{\sigma})\\
   &\overset{\eqref{CovThTemp5}}{=}\sigma\bK{\rm diag}(\by)( {\rm diag}(\by)\bK\bc^*-\frac{\boldsymbol{\lambda}^*}{\sigma}),
   \end{align*}
   which implies that
   \begin{align}\label{CovThTemp8}
      \bK\bc^*+\bK{\rm diag}(\by)\boldsymbol{\lambda}^*={\bf 0}.
   \end{align}
   By the equations \eqref{CovThTemp8}, \eqref{CovThTemp7},  \eqref{CovThTemp5} and \eqref{CovThTemp6}, we  know that
   $(\bc^*,b^*,\bu^*)$ is a proximal stationary point with
    $\gamma=\frac{1}{\sigma}$ and also a locally optimal solution to the problem \eqref{SVM01eq2} by Theorem \ref{OptiTh1}.
  \end{proof}

  \section{Numerical Experiments}
  In this section, we will show the sparsity, robustness and effectiveness of the
  proposed \textsf{$\ell_0$-KSVM} by using Jupyter Notebook on a laptop of 16GB of memory and Inter Core i7 2.30 Ghz CPU.
  %All Python codes, datasets, and numerical results can be found at \url{https://github.com/Rongrong-Lin/L0KSVM}.

  Numerous variants of SVM have been developed for binary classification problems.
  However, we do not tend to compare the $\ell_0$-KSVM with all of the existing ones. We mainly consider the KSVM with hinge loss and squared hinge loss \eqref{SHSVM} as the baseline since it is the default kernel SVM algorithm in the sklearn package.

   Similar to Theorem \ref{RepTh}, the solution $\bw\!\in\!\cF$ to KSVM with hinge loss or squared hinge loss \eqref{SHSVM} has the form $\bw=\sum_{i=1}^m c_i\phi(\bx_i)$ for some constants $c_i$, $i=1,2,\dots,m$. Then, \eqref{SHSVM} can be rewritten in the matrix form as follows:
  \begin{equation}\label{ell2SVM}
   \min_{\bc\in\bR^m, b\in \bR}\frac{1}{2}\bc^{\top}\bK\bc+C\|({\bf 1}-\diag(\by)\bK\bc-b\by)_+\|.
   \end{equation}
  For this reason, we call \eqref{ell2SVM} the \textsf{$\ell_1$-KSVM}. Similarly, the $\textsf{$\ell_2$-KSVM}$ can be expressed as follows: 
  \begin{equation}\label{ell1SVM}
   \min_{\bc\in\bR^m, b\in \bR}\frac{1}{2}\bc^{\top}\bK\bc+C\|({\bf 1}-\diag(\by)\bK\bc-b\by)_+\|^2_2.
   \end{equation}
   
  The following experiments are not to claim that $\ell_0$-KSVM is better than $\ell_1$-KSVM or $\ell_2$-KSVM, which is surely not true since the numerical performance is problem-dependent. We will see that $\ell_0$-KSVM can achieve comparable accuracy compared with $\ell_1$-KSVM and $\ell_2$-KSVM while the former generally has fewer support vectors.

  According to Theorem \ref{OptiTh1} and Definition \ref{PstationPointDef}, the proximal stationary point can be taken as a stopping criteria in our experiments. Specifically, the proposed Algorithm \ref{algo:KSVM} can be terminated if the solution $(\bc^k,b^k,\bu^k,\boldsymbol{\lambda}^k)$ at the $k$-th iteration satisfies
  \begin{equation}\label{Proximalcriteria}
  \max\{\beta_1,\beta_2,\beta_3,\beta_4\}<\varepsilon
  \end{equation}
  where $\varepsilon$ is a given tolerance level (by default $\varepsilon=10^{-3}$) and
  $$
  \begin{array}{ll}
  \beta_1&:= \frac{\|\bc^k+\diag(\by)\boldsymbol{\lambda}^k\|}{1+\|\bc^k\|+\|\boldsymbol{\lambda}^k\|}\\
  \beta_2&:=\frac{\langle \by,\boldsymbol{\lambda}^k\rangle}{m}\\
  \beta_3&:=\frac{\|\bu^k+\diag(\by)\bK\bc^k+b^k\by-{\bf 1}\|}{\sqrt{m}}\\
  \beta_4&:=\frac{\|\bu^k-\text{Prox}_{\frac{C}{\sigma}\|(\cdot)_+\|_0}(\bu^k-\boldsymbol{\lambda}^k/\sigma)\|}{1+\|\bu^k\|}.
  \end{array}
  $$
  The accuracy with respect to the data $\{(\bx^t_j,y^t_j):j\in\bN_{m_t}\}$ is defined by
  $$
  {\rm Acc}:=1-\frac{1}{2m_t}\sum_{j=1}^{m_t}|h(\bx^t_j)-y_j^t|
  $$
  where $h$ is the learned $\ell_0$-KSVM, $\ell_1$-KSVM or $\ell_2$-KSVM classifier.

  To evaluate binary classification performance, we take five evaluation criteria into consideration. They are the training accuracy (Train Acc), the testing accuracy (Test Acc), the number of support vectors (NSV), the CPU time (CPU), and the number of iterations (Iter).

  In our experiments, we set $\sigma\in\{1,2\}$, $\iota=1$ and the penalty  parameter
  \(
  C\in\{1/2,1,2,4,8,16,32,64\}.
  \)
  The label noise rate $r\!\in\!\{0\%,1\%,5\%,10\%\}$, that is, $r*2*m$ samples are flipped their labels. The maximum iteration number is set to be $2000$. The Gaussian kernel
  $$
  K(\bz,\bz')=\exp(-\rho\|\bz-\bz'\|^2),\ \bz,\bz'\in\bR^d
  $$
  is chosen as the kernel function for $\ell_0$-KSVM, $\ell_1$-KSVM and $\ell_2$-KSVM, where $\rho$ is chosen to be $1/d$.

  We conduct numerical experiments on both synthetic data and real data. The synthetic datasets used are Double Moons and Double Circles. Real datasets include Breast Cancer, Australian, Diabetes, Heart, Fourclass, German.numer, and svmguide1.
  Double Moons, Double Circles, and Breast Cancer are generated by sklearn.datasets, and the others come from
  \url{https://www.csie.ntu.edu.tw/~cjlin/libsvmtools/datasets/}.  A preprocessing step is required to scale inputs to the standard normal. For any dataset, 60\% of the samples are used for training and the rest used for testing.

  The $\ell_0$-KSVM, $\ell_1$-KSVM and $\ell_2$-KSVM classifiers for Double Circles and Double Moons are depicted in Figure \ref{Fig1} and
  Figure \ref{Fig2}, respectively. The support vectors are marked in a red circle. Remember that Remark \ref{remarksv} tells us that the support vectors must fall into the canonical hyperplanes. This is verified by our numerical experiments, see, Figures \ref{Fig1} a)-c) and  Figures \ref{Fig2} a)-c). Also, we see that $\ell_0$-KSVM is more robust to the label noise.
  The five evaluation results for comparison between
  $\ell_0$-KSVM, $\ell_1$-KSVM and $\ell_2$-KSVM for two synthetic datasets are given in Table \ref{Tab1} and Table \ref{Tab2}, respectively.

  In Table \ref{Tab3}, we present numerical results for seven real datasets. Clearly,
  $\ell_0$-KSVM achieves a comparable accuracy with fewer support vectors compared with $\ell_1$-KSVM and $\ell_2$-KSVM. The computation of $\ell_0$-KSVM is also efficient enough.

\begin{figure*}[!t]
	\centering
	\setcounter {subfigure} {0} a){
		\includegraphics[scale=0.23]{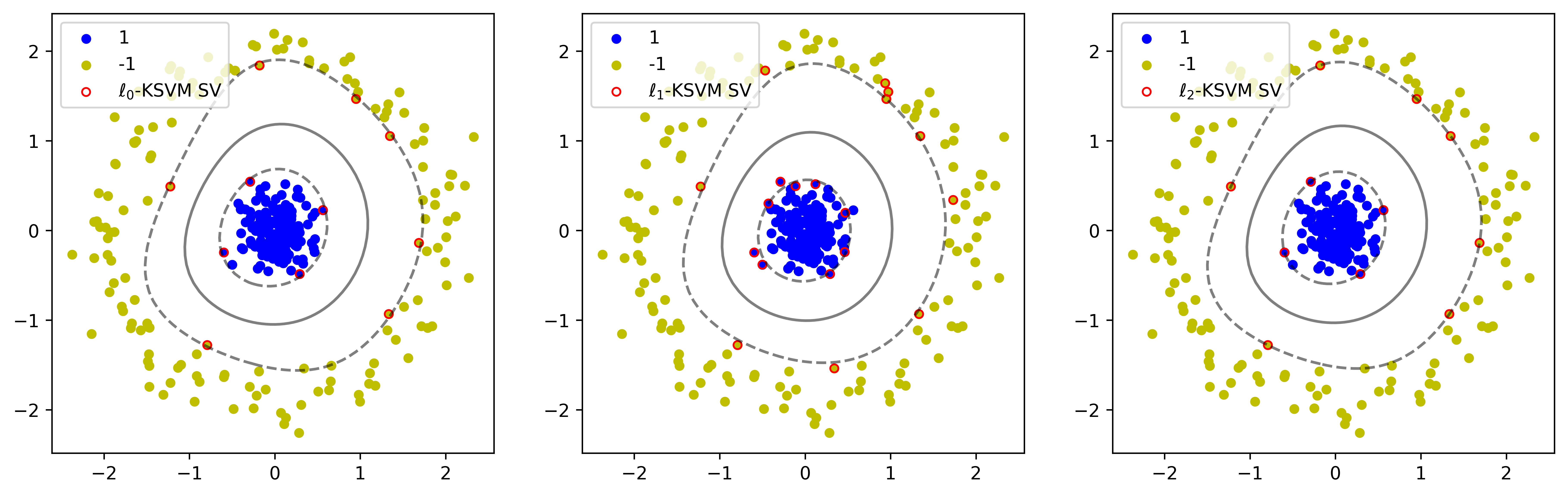}}
	\setcounter {subfigure} {0} b){
		\includegraphics[scale=0.23]{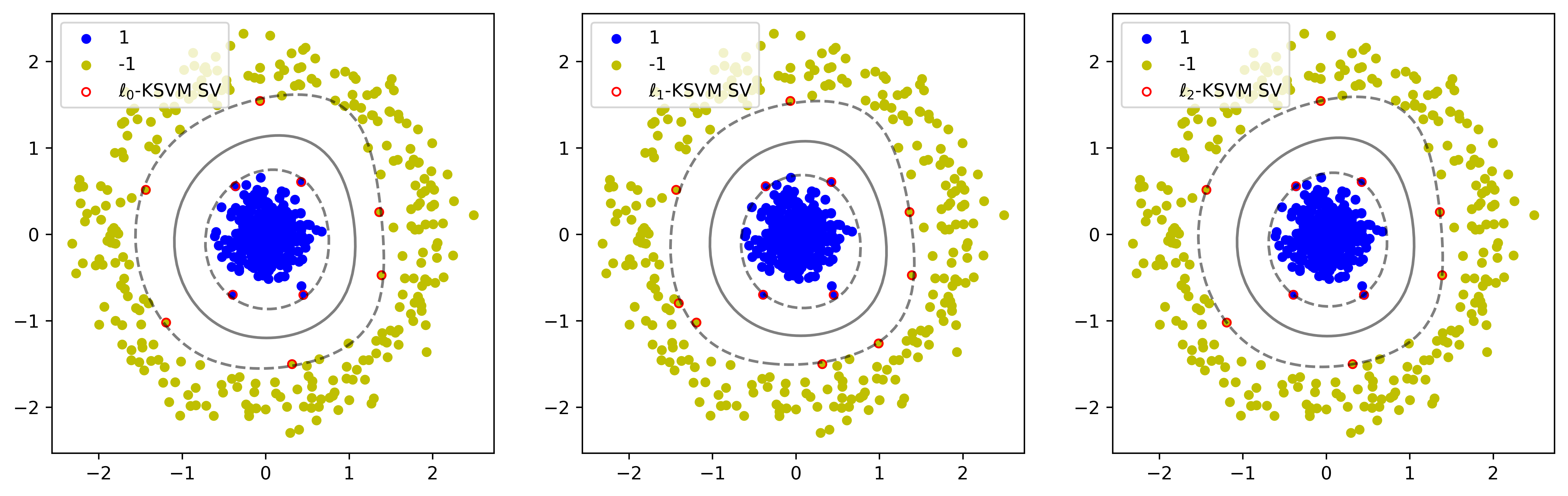}}\\
	\setcounter {subfigure} {0} c){
		\includegraphics[scale=0.23]{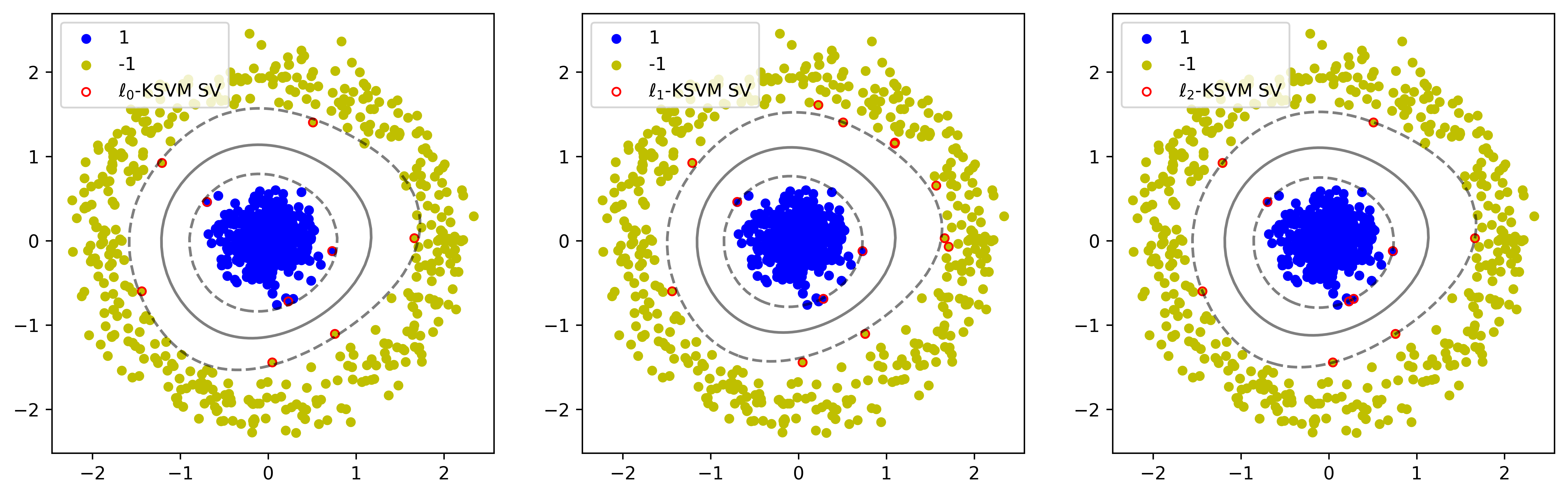}}
   \setcounter {subfigure} {0} d){
		\includegraphics[scale=0.23]{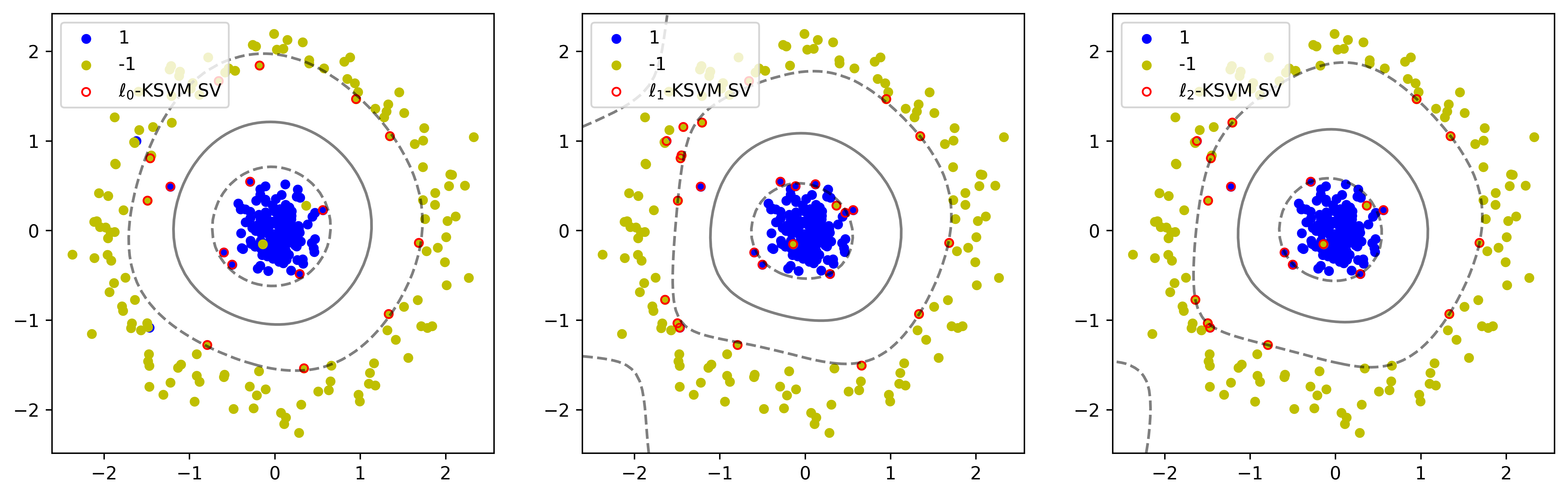}}\\
	\setcounter {subfigure} {0} e){
		\includegraphics[scale=0.23]{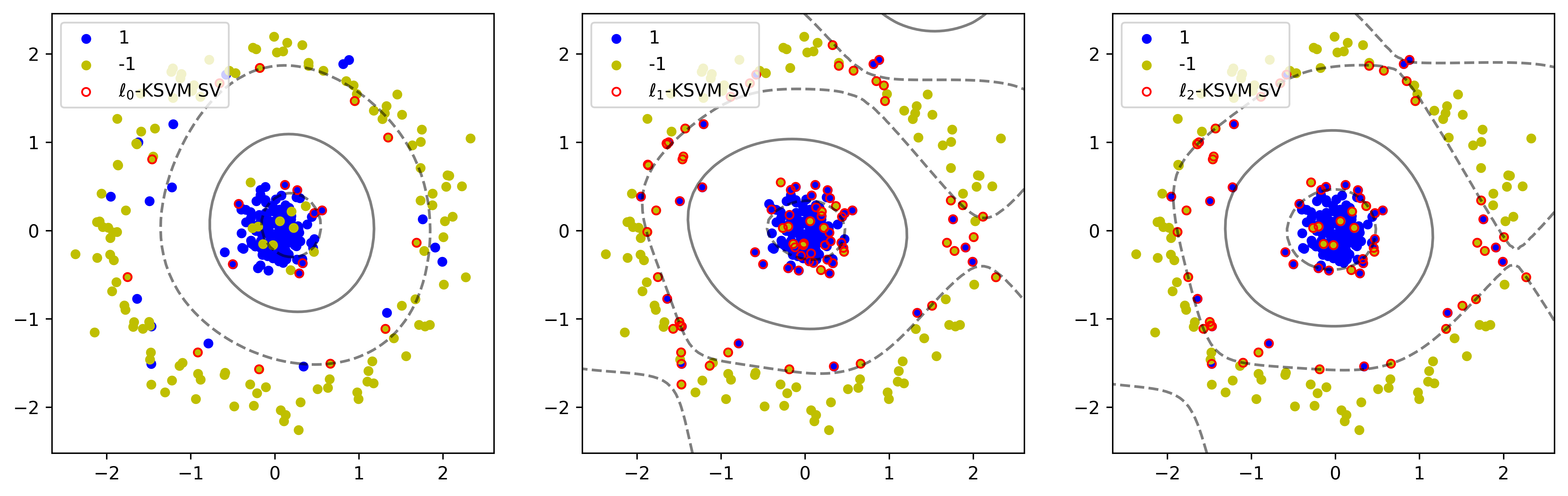}}
   \setcounter {subfigure} {0} f){
		\includegraphics[scale=0.23]{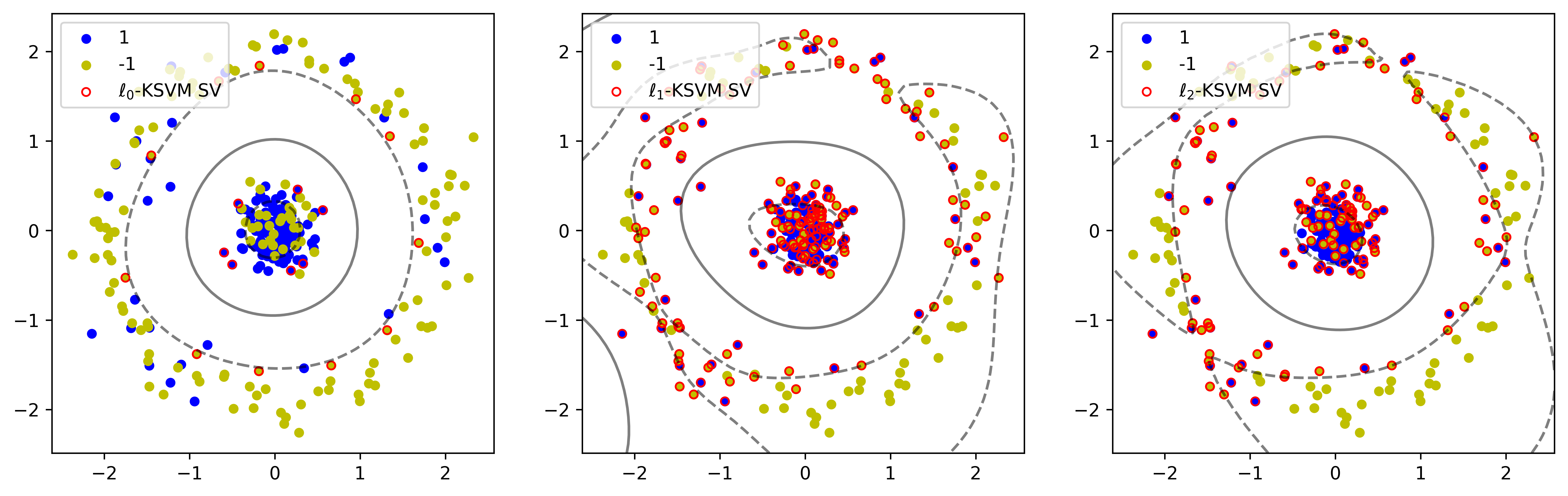} }
	\caption{Double Circles: a) $m=500$ and $r=0\%$; b) $m=1000$ and $r=0\%$; c) $m=1500$ and $r=0\%$; d) $m=500$ and $r=1\%$; e) $m=500$ and $r=5\%$; f) $m=500$ and $r=10\%$.}
	\label{Fig1}
\end{figure*}

\begin{table}[!t]
\renewcommand{\arraystretch}{1.3}
\centering
  \caption{Comparisons between $\ell_0$-KSVM, $\ell_1$-KSVM and $\ell_2$-KSVM for Double Circles}
\begin{tabular}{|c|c|c|c|c|c|c|c|}
  \hline
   $m$&r&Algorithm& Train Acc &  Test Acc & NSV & CUP(s) & Iter \\ \hline
   500&0\% &$\ell_0$-\textsf{KSVM} & 100\% & 100\%  & {\bf 11} & 0.292  & 300\\   %parameter C=8, sigma=1
      & &$\ell_1$-\textsf{KSVM}& 100\%  & 100\%  & 19 &  0.001 & 31\\ 
      & &$\ell_2$-\textsf{KSVM}& 100\%  & 100\%  & {\bf 11} &  0.001 & 2\\ \hline
   1000&0\% &$\ell_0$-\textsf{KSVM} & 100\% & 100\%  & {\bf 10} & 5.07  & 600 \\  %parameter C=8, sigma=1
     & &$\ell_1$-\textsf{KSVM}& 100\%  & 100\%  & 12 &  0.001 & 43\\ 
      & &$\ell_2$-\textsf{KSVM}& 100\%  & 100\%  & {\bf 10} &  0.001 & 2\\ \hline
   1500&0\% &$\ell_0$-\textsf{KSVM} & 100\% & 100\%  & {\bf 9} & 16.664  & 900 \\  %parameter C=8, sigma=1
     & &$\ell_1$-\textsf{KSVM}& 100\%  & 100\%  & 14 &  0.001 & 29\\ 
      & &$\ell_2$-\textsf{KSVM}& 100\%  & 100\%  & 10 &  0.001 & 2\\ \hline
    500&1\% &$\ell_0$-\textsf{KSVM} & 98.33\% & 97.5\%  & {\bf 16} & 2.281  & 2000 \\  %parameter C=1, sigma=1
      & &$\ell_1$-\textsf{KSVM}& 98.33\%  & 94.5\%  & 29 &  0.222 & 77\\
      & &$\ell_2$-\textsf{KSVM}& 98.33\%  & 97\%  & 20 &  0.001 & 2\\ \hline
      
    500&5\% &$\ell_0$-\textsf{KSVM} & 90.67\% & 89\%  & {\bf 19} & 2.208  & 2000 \\  %parameter C=0.5, sigma=1
     & &$\ell_1$-\textsf{KSVM}& 91\%  & 88.5\%  & 94 &  0.215 & 100\\
     & &$\ell_2$-\textsf{KSVM}& 90.67\%  & 89\%  & 73 &  0.001 & 2\\ \hline
     
    500&10\% &$\ell_0$-\textsf{KSVM} & 79.33\% & 81\%  & {\bf 18} & 2.381  & 2000 \\  %parameter C=1, sigma=2
    & &$\ell_1$-\textsf{KSVM}& 79.67\%  & 81\% & 179  &  0.240 & 100\\
    & &$\ell_2$-\textsf{KSVM}& 79.33\%  & 81\% & 141  &  0.003 & 2\\ \hline
\end{tabular}
  \label{Tab1}
\end{table}

\begin{figure*}[!t]
	\centering
	\setcounter {subfigure} {0} a){
		\includegraphics[scale=0.23]{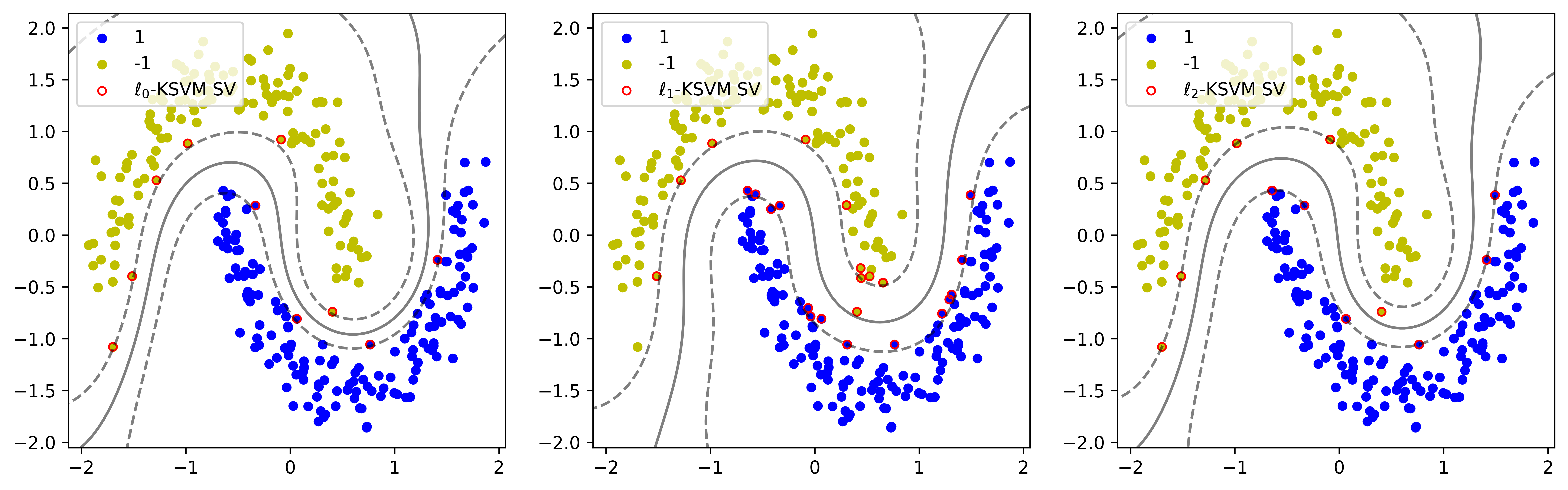}}
	\setcounter {subfigure} {0} b){
		\includegraphics[scale=0.23]{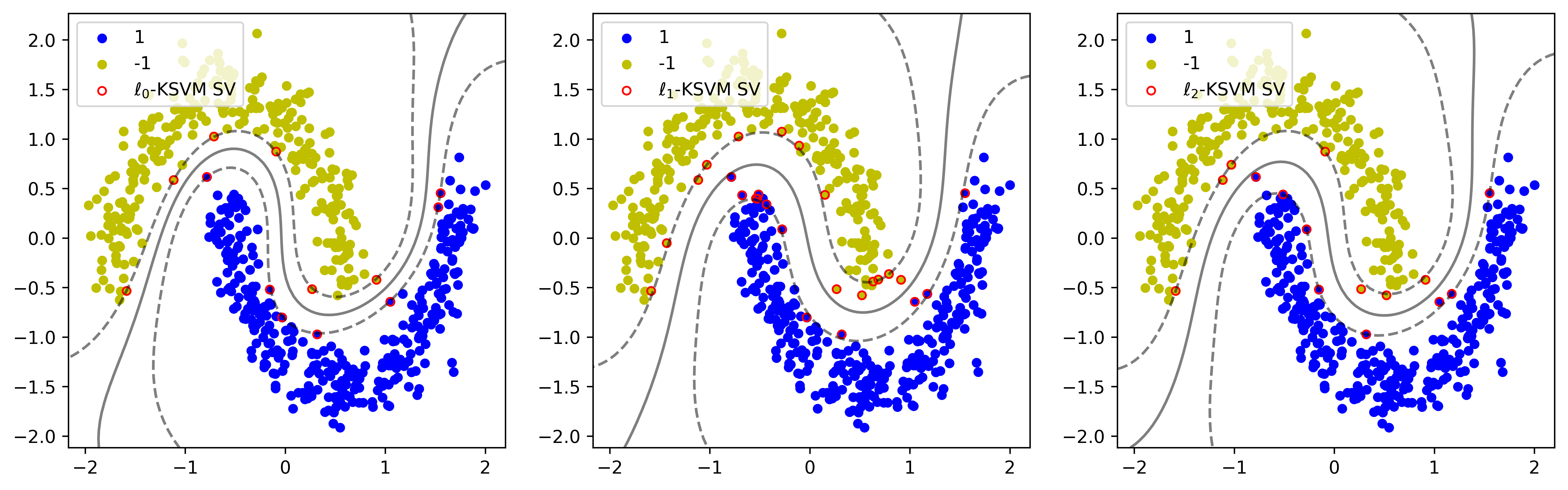}}\\
	\setcounter {subfigure} {0} c){
		\includegraphics[scale=0.23]{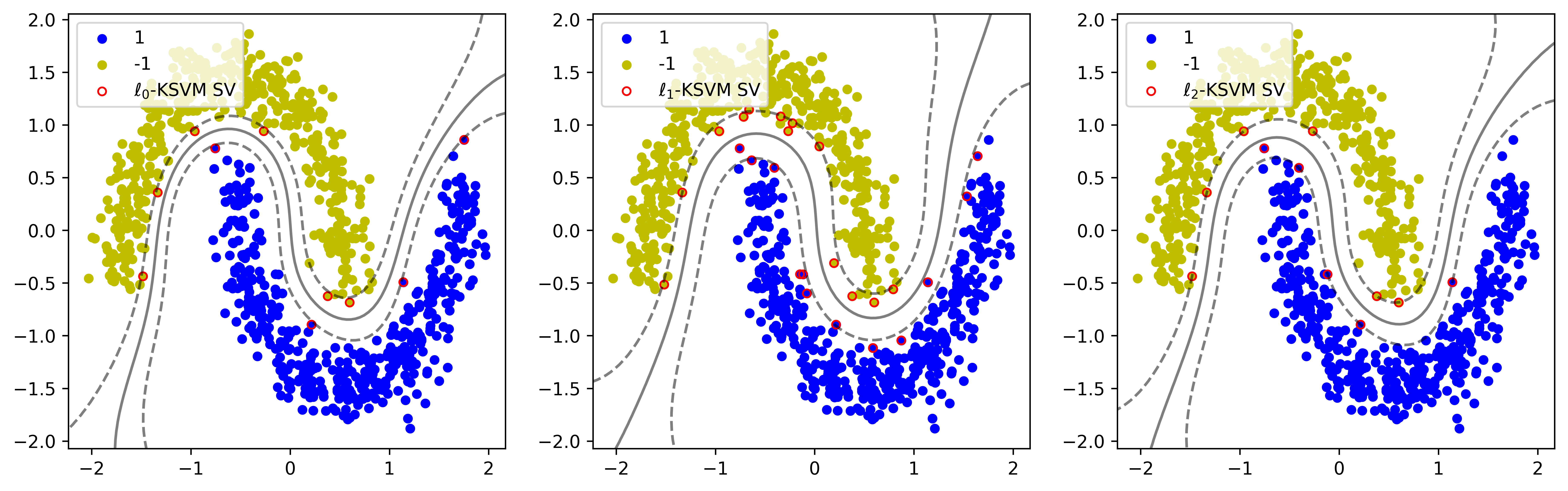}}
   \setcounter {subfigure} {0} d){
		\includegraphics[scale=0.23]{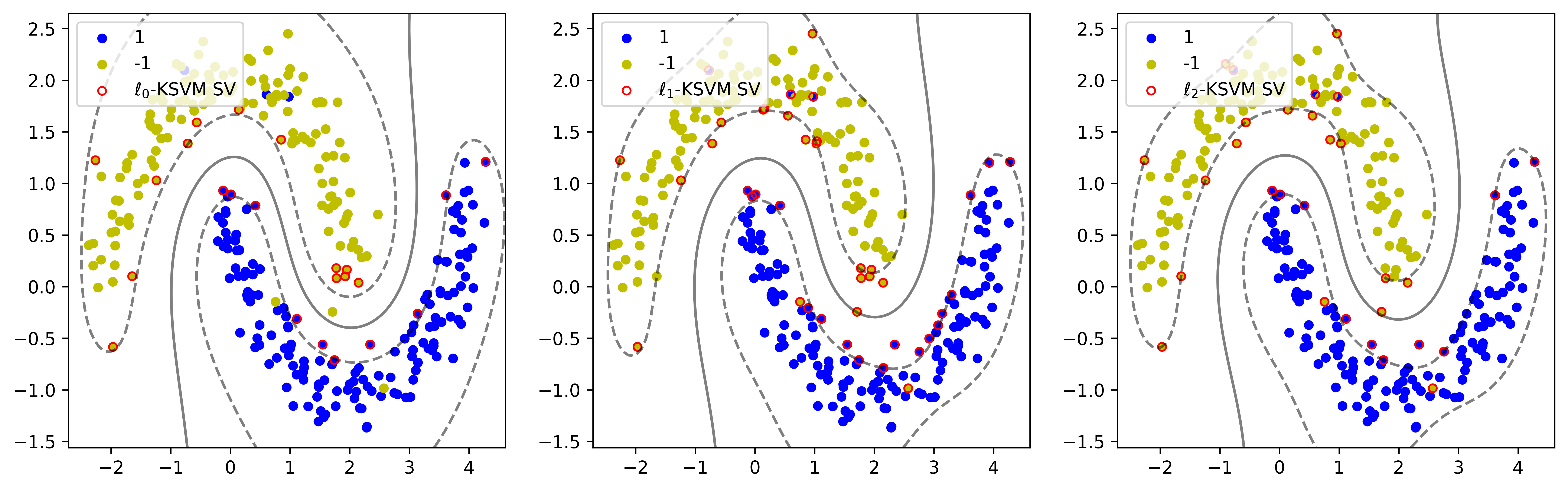} }\\
	\setcounter {subfigure} {0} e){
		\includegraphics[scale=0.23]{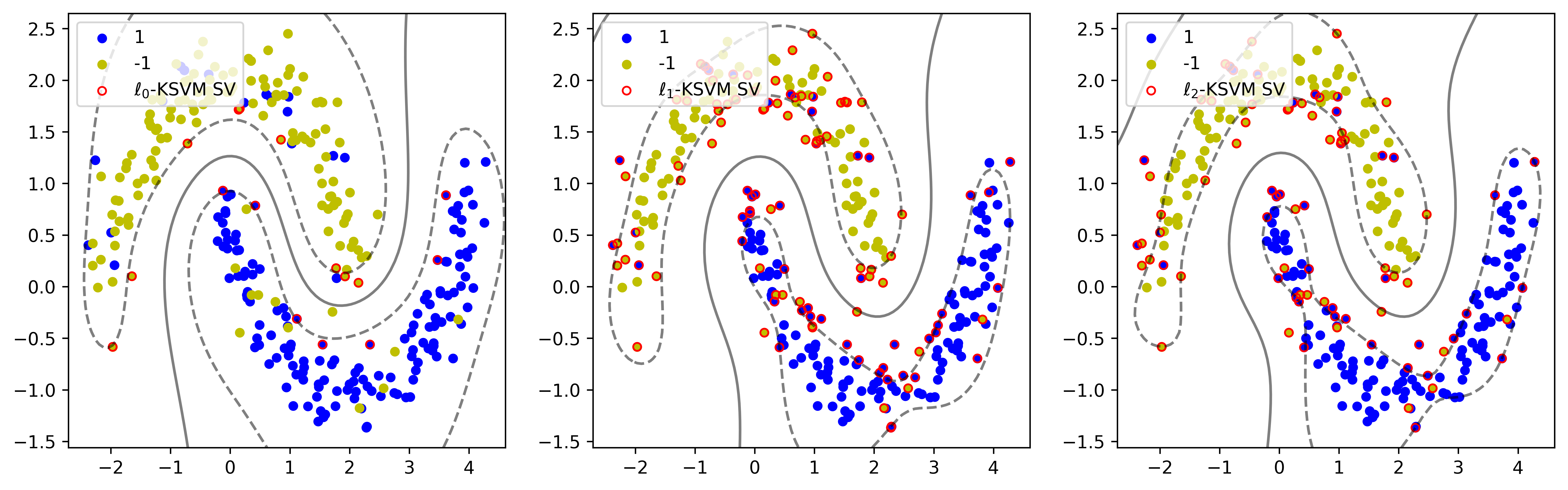}}
   \setcounter {subfigure} {0} f){
		\includegraphics[scale=0.23]{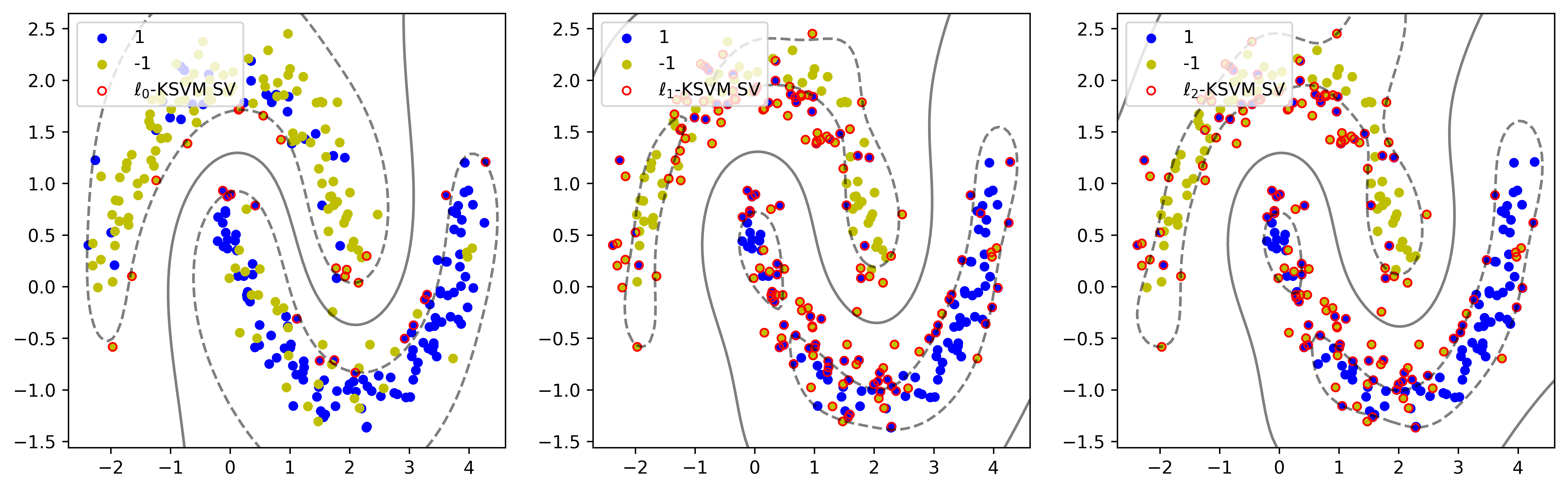} }\\
	\caption{Double Moons: a) $m=500$ and $r=0\%$; b) $m=1000$ and $r=0\%$; c) $m=1500$ and $r=0\%$; d) $m=500$ and $r=1\%$; e) $m=500$ and $r=5\%$; f) $m=500$ and $r=10\%$.}
	\label{Fig2}
\end{figure*}

\begin{table}[!t]
\renewcommand{\arraystretch}{1.3}
\centering
  \caption{Comparisons between $\ell_0$-KSVM, $\ell_1$-KSVM and $\ell_2$-KSVM  for Double Moons}
\begin{tabular}{|c|c|c|c|c|c|c|c|}
  \hline
   $m$&r&Algorithm& Train Acc &  Test Acc & NSV & CUP(s) & Iter \\ \hline
   500&0\% &$\ell_0$-\textsf{KSVM} & 100\% & 99.5\%  & {\bf 10} & 0.394  & 330\\
     & &$\ell_1$-\textsf{KSVM}& 100\%  & 99.5\%  & 24 &  0.219 & 100\\
      & &$\ell_2$-\textsf{KSVM}& 100\%  & 100\%  & 12 &  0.001 & 2\\ \hline
      
   1000&0\% &$\ell_0$-\textsf{KSVM} & 100\% & 99.5\%  & {\bf 13} & 5.405  & 562\\
     & &$\ell_1$-\textsf{KSVM}& 100\%  & 99.25\%  & 26 &  0.462 & 100\\
      & &$\ell_2$-\textsf{KSVM}& 100\%  & 99.75\%  & 15 &  0.004 & 2\\ \hline
      
   1500&0\% &$\ell_0$-\textsf{KSVM} & 100\% & 100\%  & {\bf 10} & 14.845  & 628\\
     & &$\ell_1$-\textsf{KSVM}& 100\%  & 100\%  & 25 &   0.684 & 100\\
      & &$\ell_2$-\textsf{KSVM}& 100\%  & 100\%  & 11 &  0.004 & 2\\ \hline

   500&1\% &$\ell_0$-\textsf{KSVM} & 98.33\% & 97.5\%  & {\bf 23} & 2.171  & 2000\\
     & &$\ell_1$-\textsf{KSVM}& 98.33\%  & 97\%  & 41 &  0.062 & 32\\
      & &$\ell_2$-\textsf{KSVM}& 98.33\%  & 97\%  & 30 &  0.004 & 2\\ \hline

   500&5\% &$\ell_0$-\textsf{KSVM} & 90.33\% & 88.5\%  & {\bf 16} & 2.048  & 2000\\
     & &$\ell_1$-\textsf{KSVM}& 90.33\%  & 88.5\%  & 106 &  0.156 & 63\\
      & &$\ell_2$-\textsf{KSVM}& 90.33\%  & 88.5\%  & 81 &  0.005 & 2\\ \hline

   500&10\% &$\ell_0$-\textsf{KSVM} & 79.33\% & 81\%  & {\bf 27} & 2.119  & 2000\\
     & &$\ell_1$-\textsf{KSVM}& 79.33\%  & 80.5\%  & 166 &  0.219 & 100\\
      & &$\ell_2$-\textsf{KSVM}& 79.33\%  & 81\%  & 148 &  0.015 & 2\\ \hline
\end{tabular}
  \label{Tab2}
\end{table}

\begin{table}[!t]
\renewcommand{\arraystretch}{1.3}
\centering
  \caption{Comparisons between $\ell_0$-KSVM, $\ell_1$-KSVM and $\ell_2$-KSVM for seven real datasets}
\begin{tabular}{|c|c|c|c|c|c|c|c|c|}
  \hline
   Datasets &Algorithm& Train Acc &  Test Acc & NSV & CUP(s) & Iter \\ \hline
   Breast Cancer&$\ell_0$-\textsf{KSVM} & 92.38\% & 97.37\%  & {\bf 79} & 0.116  & 17\\
   ($m\!=\!569,d\!=\!30$)  &$\ell_1$-\textsf{KSVM}& 98.53\%  & {\bf 98.68}\%  & 87 &  0.294 & 100\\
   &$\ell_2$-\textsf{KSVM}& 99.12\%  & 98.25\%  & 83 &  0.007 & 2\\ \hline

   Australian&$\ell_0$-\textsf{KSVM} & 88.41\% &  83.33\% & {\bf 105} &  0.497 &45 \\
   ($m\!=\!690,d\!=\!14$)&$\ell_1$-\textsf{KSVM}& 91.79\%  & 83.33\%  & 184 &  0.250 & 100\\
   & $\ell_2$-\textsf{KSVM}& 91.79\% & {\bf 84.06}\%  & 188 &  0.012 & 2\\ \hline

   Diabetes  &$\ell_0$-\textsf{KSVM} & 79.13\% & {\bf 78.25}\%  &{\bf 98}  & 1.656 &191 \\
    ($m\!=\!768,d\!=\!8$)&$\ell_1$-\textsf{KSVM}& 82.39\%  & 75\%  & 275 &  0.197 & 100\\
    & $\ell_2$-\textsf{KSVM}& 81.52\% & 77.27\%  & 287 & 0.008  & 2\\ \hline

     Heart  &$\ell_0$-\textsf{KSVM} & 88.88\% & {\bf 85.19}\%  &{\bf 76}  & 0.066 & 58\\
    ($m\!=\!270,d\!=\!13$)&$\ell_1$-\textsf{KSVM}& 91.98\%  & 82.41\%  & 105 &  0.016 & 34\\
    & $\ell_2$-\textsf{KSVM}& 91.98\% & 84.26\%   & 97 & 0.003  & 2\\ \hline

     Fourclass  &$\ell_0$-\textsf{KSVM} & 100\% & {\bf 100}\%  &{\bf 85}  & 4.292 & 301\\
    ($m\!=\!862,d\!=\!2$)&$\ell_1$-\textsf{KSVM}& 99.42\%  & 99.71\%  & 130 &  0.375 & 100\\
    & $\ell_2$-\textsf{KSVM}& 98.07\% & 99.13\%   & 151 & 0.008  & 2\\ \hline

    German.numer  &$\ell_0$-\textsf{KSVM} & 86.33\% & {\bf 77.25}\%  &{\bf 284}  &10.418 & 626\\
    ($m\!=\!1000,d\!=\!24$)&$\ell_1$-\textsf{KSVM}& 94\%  & 77\%  & 379 &  0.458 & 100\\
    & $\ell_2$-\textsf{KSVM}& 84.5\% & {\bf 77.25}\%   & 418 & 0.025 & 2\\ \hline

     svmguide1  &$\ell_0$-\textsf{KSVM} & 94.82\% &  94.58\%  &{\bf 171}  &22.48 &118 \\
    ($m\!=\!3089,d\!=\!4$)&$\ell_1$-\textsf{KSVM}& 97.36\%  & 96.85\%  & 209 &  0.966 & 100\\
    & $\ell_2$-\textsf{KSVM}& 97.25\% & {\bf 97.17}\%   & 293 & 0.031 & 2\\ \hline
\end{tabular}
\label{Tab3}
\end{table}

   \section{Conclusion}

   In this paper, we develop the theoretical and algorithmic analysis for the kernel SVM with $\ell_0$-norm hinge loss, namely \textsf{$\ell_0$-KSVM}. Theoretically, we have proved the equivalent relationship among the proximal stationary point, Karush-Kuhn-Tucker point, and the local minimizer of $\ell_0$-KSVM by the limiting subdifferential of the $\ell_0$-norm hinge loss. Algorithmically, we have proposed the ADMM algorithm for $\ell_0$-KSVM and presented its convergence analysis. Finally, the numerical experiments on the synthetic and real datasets have justified our theoretical and algorithmic analysis.

  \appendices

 \section{Proof of Lemma \ref{subdiff-Cgz}} \label{AppendixA}
  \begin{proof} {\it (i)}. Fix any $\delta>0$. Then $L_{0/1}(g(\bz))\!\geq \!L_{0/1}(g(\overline{\bz}))$ for all
   $\bz\in\! U_{\delta}(\overline{\bz})$ by the definition of $L_{0/1}$
   and the continuity of $g$.
   We consider the following two cases.

   \medskip
   \noindent
   Case 1: $L_{0/1}(g(\bz))\!>\!L_{0/1}(g(\overline{\bz}))$
   for all $\bz\in\!U_{\delta}(\overline{\bz})\backslash\{\overline{\bz}\}$.
   Now $\Theta_{\overline{I}}\backslash\{\overline{\bz}\}=\emptyset$.
   By the definition of regular normal,
   $\widehat{\mathcal{N}}_{\Omega_{\overline{I}}}(\overline{\bz})=\mathbb{R}^l$.
   In addition, for any $\bv\in\mathbb{R}^l$,
   $$
   \begin{array}{ll}
    &\displaystyle{\liminf_{\substack{\bz\to\overline{\bz}\\ \bz\ne\overline{\bz}}}
    \frac{\varphi(\bz)-\varphi(\overline{\bz})-\langle \bv,\bz-\overline{\bz}\rangle}{\|\bz-\overline{\bz}\|} }\\
    &\displaystyle{=\liminf_{\substack{\bz\to\overline{\bz}\\ \bz\ne\overline{\bz}}}
    \frac{C L_{0/1}(g(\bz))-C L_{0/1}(g(\overline{\bz}))-\langle \bv,\bz-\overline{\bz}\rangle}
    {\|\bz-\overline{\bz}\|}\ge 0, }
    \end{array}
   $$
   which means that $\bv\in\widehat{\partial}\varphi(\overline{\bz})$.
   By the arbitrariness of $\bv$ in $\mathbb{R}^l$, we have that
   $\widehat{\partial}\varphi(\overline{\bz})=\mathbb{R}^l$.
   Consequently, $\widehat{\partial}\varphi(\overline{\bz})
           =\widehat{\mathcal{N}}_{\Theta_{\overline{I}}}(\overline{\bz})=\mathbb{R}^l$
                and  then $\partial\varphi(\overline{\bz})=\mathcal{N}_{\Theta_{\overline{I}}}(\overline{\bz})=\mathbb{R}^l$.
  So, the two equalities in Statement (i) hold for this case.

  \medskip
  \noindent
   Case 2: there exists
   $\bz_{\delta}\!\in\! U_{\delta}(\overline{\bz})\backslash\{\overline{\bz}\}$
   with  $L_{0/1}(g(\bz))\!=\!L_{0/1}(g(\overline{\bz}))$.
   Fix any $\bv\in\mathbb{R}^l$. Then, it holds that
   \begin{align*}
    &\liminf_{\substack{\bz\to\overline{\bz}\\ \bz\ne\overline{\bz}}}
    \frac{\varphi(\bz)-\varphi(\overline{\bz})-\langle \bv,\bz-\overline{\bz}\rangle}{\|\bz-\overline{\bz}\|} \\
    &=\liminf_{\substack{\bz\to\overline{\bz}\\ \bz\neq \overline{\bz}}}
    \frac{C L_{0/1}(g(\bz))-C L_{0/1}(g(\overline{\bz}))-\langle \bv,\bz-\overline{\bz}\rangle}
    {\|\bz-\overline{\bz}\|}\\
    &=\liminf_{\substack{\bz\to\overline{\bz},\bz \neq \overline{\bz}\\{\rm ps}(g(\bz))=I}}
    \frac{-\langle \bv,\bz-\overline{\bz}\rangle}{\|\bz-\overline{\bz}\|}
    =\liminf_{\substack{\bz\ne\overline{\bz}\\ \bz\xrightarrow[\Theta_{\overline{I}}]{}\overline{\bz}}}
    \frac{-\langle \bv,\bz-\overline{\bz}\rangle}{\|\bz-\overline{\bz}\|}.
   \end{align*}
   By the definition of the regular subgradient, we have
    $\widehat{\partial}\varphi(\overline{\bz})
    =\widehat{\mathcal{N}}_{\Theta_{\overline{I}}}(\overline{\bz})$.
   Together with Case 1, we obtain the first  equality in Statement (i).

   Next we show that $\partial\varphi(\overline{\bz})=\mathcal{N}_{\Theta_{\overline{I}}}(\overline{\bz})$.
   Pick any $\bv\in\partial\varphi(\overline{\bz})$. There exist sequences
   $\bz^k\xrightarrow[\varphi]{}\overline{\bz}$ and $\bv^k\in\widehat{\partial}\varphi(\bz^{k})$
   with $\bv^k\to \bv$. Since $\bz^k\xrightarrow[\varphi]{}\overline{\bz}$, for all $k$ large enough,
   we have $L_{0/1}(g(\bz^k))=L_{0/1}(g(\overline{\bz}))$
   and then ${\rm ps}(g(\bz^k))=I$, which means that $\bz^k\in\Theta_{\overline{I}}$.
   From the first equality, $\widehat{\partial}\varphi(\bz^{k})
   =\widehat{\mathcal{N}}_{\Theta_{\overline{I}}}(\bz^k)$
   for all $k$ large enough. So, $\bv\in\mathcal{N}_{\Theta_{\overline{I}}}(\overline{\bz})$
   and $\partial\varphi(\overline{\bz})\subseteq\mathcal{N}_{\Theta_{\overline{I}}}(\overline{\bz})$.
   Conversely, pick any $\bv\in\mathcal{N}_{\Theta_{\overline{I}}}(\overline{\bz})$.
   There exist $\bz^k\xrightarrow[\Theta_{\overline{I}}]{}\overline{\bz}$
   and $\bv^k\in\widehat{\mathcal{N}}_{\Theta_{\overline{I}}}(\bz^k)$ with $\bv^k\to \bv$.
   Clearly, $\bz^k\xrightarrow[\varphi]{}\overline{\bz}$. From
   $\widehat{\partial}\varphi(\bz^{k})=\widehat{\mathcal{N}}_{\Theta_{\overline{I}}}(\bz^k)$,
   we have $\bv\in\partial\varphi(\overline{\bz})$ and
   $\mathcal{N}_{\Theta_{\overline{I}}}(\overline{\bz})\subseteq\partial\varphi(\overline{\bz})$.
   Thus, $\partial\varphi(\overline{\bz})=\mathcal{N}_{\Theta_{\overline{I}}}(\overline{\bz})$.

  \medskip
  \noindent
  {\it (ii)}. Since $\mathcal{G}$ is metrically subregular at $\overline{\bz}$ for the origin,
   by using \cite[Page 211]{ioffe2008metric} and the definition of $\Theta_{\overline{I}}$,
   we have $\mathcal{N}_{\Theta_{\overline{I}}}(\overline{\bz})
  \subseteq \nabla\!g_{\overline{I}}(\overline{x})\mathcal{N}_{\mathbb{R}_{-}^{|\overline{I}|}}$.
   Note that
   \[
   \mathcal{N}_{\Theta_{\overline{I}}}(\overline{\bz})
   \supseteq\widehat{\mathcal{N}}_{\Omega_{\overline{I}}}(\overline{x})
   \supseteq\nabla\!g_{\overline{I}}(\overline{x})\mathcal{N}_{\mathbb{R}_{-}^{|\overline{I}|}}
   \]
   by \cite[Theorem 6.14]{1998Variational}. Therefore, the desired result (3.2) can be obtained.
  \end{proof}

   \section{Proof of Theorem \ref{OptiTh1}} \label{AppendixB}

   \begin{proof}
    Denote the feasible set of the problem (2.4) by
    $\Omega$, i.e.,
    \[
    \Omega:=\big\{(\bc,b,\bu):\ \bu+{\rm diag}(\by)\bK \bc+b\by ={\bf 1}\big\}.
    \]

   We first prove the necessity.
   Let $\bv^*\!:=\!(\bc^*,b^*,\bu^*)$ be a proximal stationary point of (2.4), then there exists $\boldsymbol{\lambda}^*\in \bR^{m}$ such that $(\bv^*,\boldsymbol{\lambda}^*)$ satisfies
   equations (3.6a)-(3.6d) and (3.5) with
   ${\boldsymbol\eta}^*=\bu^*-\gamma \boldsymbol{\lambda}^*$.
   We will argue that $\bv^*$ is a locally optimal to (2.4), i.e., there exists $\delta>0$ such that for any $\bv\in \Omega\cap U_{\delta}(\bv^*)$
   \begin{align}\label{OptiTh1Temp1}
   \frac{1}{2}(\bc^*)^{\top}\bK\bc^*+C\|\bu^*_+\|_0
   \leq \frac{1}{2}\bc^{\top}\bK\bc+C\|\bu_+\|_0.
   \end{align}
   As $\|\bu_+\|_0$ is lower semi-continuous at $\bu^*$, by \cite[Proposition 4.3]{Mordukhovich2013},
   there exists $\delta_1>0$ such that
   \begin{align}\label{OptiTh1Temp2}
       \|\bu^*_+\|_0\leq \|\bu_+\|_0+1/2,\mbox{ for any } \bv\in \Omega\cap U_{\delta_1}(\bv^*).
   \end{align}
   Notice that $\|\bu_+\|_0$ can only take values from the set
   $\{0,1,\dots,m\}$. Together with \eqref{OptiTh1Temp2}, we can know that
   \begin{align}\label{OptiTh1Temp3}
       \|\bu^*_+\|_0\leq \|\bu_+\|_0,\mbox{ for any } \bv\in \Omega\cap U_{\delta_1}(\bv^*).
   \end{align}
   Define two index sets
   \begin{align*}
   \Gamma_*&=\{i\in \bN_m: u^*_i=0\} {\quad\rm and\quad}
   \overline{\Gamma}_*=\bN_m\!\setminus\! \Gamma_*.
   \end{align*}
   Together with (3.6d) and (3.5), we know that
   \begin{align}\label{GammaDefeq1}
       u^*_i=0,\;\; -\sqrt{2 C/\gamma}\leq \lambda^*_i< 0,\mbox{ for any } i\in \Gamma_*
   \end{align}
   and
   \begin{align}\label{GammaDefeq2}
      u^*_i\neq 0,\quad \lambda^*_i=0,\mbox{ for any } i\in \overline{\Gamma}_*.
   \end{align}
   Denote
   \[
   \Omega_1:=\{\bv\in \Omega: u_i\leq 0, i\in \Gamma_*\}.
   \]
   By \eqref{OptiTh1Temp3} and definition of $\|\bu_+\|_0$, we can conclude that
   \begin{align}\label{OptiTh1Temp4}
       \|\bu^*_+\|_0+1\leq \|\bu_+\|_0,\mbox{ for any } \bv\in (\Omega\!\setminus\!\Omega_1)\!\cap\! U_{\delta_1}(\bv^*).
   \end{align}
   On the other hand,
    since that $\bc^{\top}\bK\bc$ is locally Lipschitz continuous in $\bR^{m}$, there exists $\delta_2>0$ such that
   \begin{align}\label{OptiTh1Cineq}
   | \bc^{\top}\bK\bc\!-\!(\bc^*)^{\top}\bK\bc^*|\leq 2C,\mbox{ for any }  \bv\!\in\! \Omega\!\cap\! U_{\delta_2}(\bv^*).
   \end{align}
   Set $\delta:=\min\{\delta_1,\delta_2\}$.
 We split the proof of the \eqref{OptiTh1Temp1} into the following two cases:

   \medskip
   Case (i) $\bv\in  \Omega_1\cap U_{\delta}(\bv^*)$.
   Then
   \begin{align}\label{OptiTh1Temp5}
   u_i\le 0, \mbox{ for any } i\in \Gamma_*.
   \end{align}
   Since $\bv,\bv^*\in \Omega$, by (3.6b), we have
   \begin{align}\label{OptiTh1Temp6}
       -{\rm diag}(\by)\bK(\bc-\bc^*)=(\bu-\bu^*)+(b-b^*)\by.
   \end{align}
   Notice that
   \begin{align}
   \begin{array}{rll}
   &&\frac{1}{2}\big(\bc^{\top}\bK\bc-(\bc^*)^{\top}\bK\bc^*\big)\nonumber\\
   &=&\frac{1}{2}(\bc-\bc^*)^{\top}\bK(\bc-\bc^*)
   +{(\bc^*)}^{\top}\bK(\bc-\bc^*)\nonumber\\
   &\geq & {\bc^*}^{\top}\bK(\bc-\bc^*)\nonumber\\
   & \overset{(3.6a)}{=} & -\langle \bK{\rm diag}(\by)\boldsymbol{\lambda}^*, \bc-\bc^*\rangle\nonumber\\
   &=& -\langle \boldsymbol{\lambda}^*, {\rm diag}(\by)\bK(\bc-\bc^*)\rangle\nonumber\\
   &\overset{\eqref{OptiTh1Temp6}}{=}& \langle \boldsymbol{\lambda}^*,(\bu-\bu^*)+(b-b^*)\by\rangle\nonumber\\
   &\overset{(3.6b)}{=}& \langle \boldsymbol{\lambda}^*,\bu-\bu^*\rangle \nonumber\\
   &\overset{\eqref{GammaDefeq1},\eqref{GammaDefeq2}}{=}&\sum\limits_{i\in \Gamma_*}(\lambda^*_i u_i),\nonumber\\
   &\overset{\eqref{GammaDefeq1},\eqref{OptiTh1Temp5}}{\geq }& 0.
   \end{array}
   \end{align}
  Together the above inequality with \eqref{OptiTh1Temp3}, we can obtain that for any $\bv\in  \Omega_1\cap U_{\delta}(\bv^*)$,
  \begin{align}\label{OptiTh1Temp7}
  \frac{1}{2}(\bc^*)^{\top}\bK\bc^*+C\|\bu^*_+\|_0
   \leq \frac{1}{2}\bc^{\top}\bK\bc+C\|\bu_+\|_0.
  \end{align}

  \medskip

   Case (ii) $\bv\in (\Omega\!\setminus\! \Omega_1)\cap U_{\delta}(\bv^*)$. Then
  \begin{align}\label{OptiTh1Temp8}
   \begin{array}{rcl}
    \frac{1}{2}\bc^{\top}\bK\bc\!+\!C\|\bu_+\|_0
    & \overset{\eqref{OptiTh1Cineq}}{\geq } &
    \frac{1}{2}(\bc^*)^{\top}\bK\bc^*\!-\!C\!+\! C\|\bu_+\|_0 \\
    &\overset{\eqref{OptiTh1Temp4}}{\geq } &
    \frac{1}{2}(\bc^*)^{\top}\bK\bc^*+\|\bu^*_+\|_0.
    \end{array}
    \end{align}
   Hence, the inequality \eqref{OptiTh1Temp1} holds from  \eqref{OptiTh1Temp7} and  \eqref{OptiTh1Temp8}. The proof of sufficiency is complete.

   \medskip

   To prove the sufficiency, we write $\bz:=(\bc,b)$ and
     \[
     g(\bz):={\bf 1}-{\rm diag}(\by)\bK \bc-b\by  \mbox{ and } \varphi(\bz):=L_{0/1}(g(\bz)).
     \]
     Suppose that $(\bc^*,b^*,\bu^*)$ is a locally optimal solution to (2.4). Hence, we have
     \begin{align*}
     \bu^*+{\rm diag}(\by)\bK \bc^*+b^*\by ={\bf 1}
     \end{align*}
     and $\bz^*:=(\bc^*,b^*)$ is a locally optimal solution to (2.3).
     Then by the optimal condition \cite[Theorem 10.1]{1998Variational} and Theorem 3.8, there exists $\boldsymbol{\lambda}^*\in\bR^{m}$ such that
     \begin{align}\label{KKTTemp}
         \begin{array}{c}
             \begin{pmatrix}
                        \bK \bc^*\\
                        {\bf 0 }
                       \end{pmatrix}+\begin{pmatrix}\bK{\rm diag}(\by)\\\by^{\top}\end{pmatrix}\boldsymbol{\lambda}^*={\bf 0}
         \end{array}
     \end{align}
  and
   \begin{align}\label{LambdaTemp}
       \lambda_i^*\in \left\{\begin{array}{cl}
        \bR_- ,   &  {\rm if\;} i\in \overline{I}_0,\\
        0 ,   &  {\rm if\;} i\in I\cup\overline{I}_{<},
       \end{array}\right.
   \end{align}
   with $I\!:=\!\{i\!\in\! \bN_m : g(\bz^*)>0\}$, $\overline{I}_0\!:=\!\{i\!\in\! \bN_m : g(\bz^*)=0\}$ and $\overline{I}_{<}\!:=\!\{i\!\in\! \bN_m : g(\bz^*)<0\}$.
   The equation \eqref{KKTTemp} implies  (3.6a) and (3.6b) both hold.
  To argue  $(\bc^*,b^*,\bu^*)$ being a proximal stationary point, it suffices to prove that ${\rm Prox}_{\gamma C L_{0/1}}(\bu^*-\gamma\boldsymbol{\lambda}^*)  =\bu^*$ for some $\gamma>0$.
  Write $\overline{I}_{0-}:=\{i\in \overline{I}_0 : \lambda^*_i<0\}$.
  Take
  \begin{equation}\label{gamma}
   \gamma\;\left\{\begin{array}{ll}
  \! =\!\min\big\{\frac{2C}{\max_{i\in \overline{I}_0}(\lambda^*_i)^2}, \frac{\min_{i\in I}(u^*_i)^2}{2C}\big\},&
   \mbox{ if }\overline{I}_{0-}\!\ne\!\emptyset, I\!\ne\!\emptyset,\\                                                                           \!=\!\frac{\min_{i\in I}(u^*_i)^2}{2C}, &\mbox{ if }\overline{I}_{0-}\!=\!\emptyset, I\!\ne\!\emptyset, \\
   \!=\!\frac{2C}{\max_{i\in \overline{I}_0}(\lambda^*_i)^2}, &\mbox{ if }\overline{I}_{0-}\!\ne\!\emptyset,I\!=\!\emptyset, \\
   \in (0,+\infty),& \mbox{ otherwise}. \\
   \end{array}\right.
   \end{equation}
  Notice that $\boldsymbol{\eta}^*:=\bu^*-\gamma\boldsymbol{\lambda}^*$ and $\bu^*=g(\bz^*)$. Then from \eqref{LambdaTemp}, it follows that for any $i\in \bN_m$
  \begin{align*}
        (\boldsymbol{\eta}^*)_i:=\left\{\begin{array}{cl}
                         -\gamma\lambda^*_i,  & {\rm if\;} i\in \overline{I}_0,\\
                         u^*_i,  & {\rm if\;} i\in \overline{I}_{<}\cup I.\\
                          \end{array}
                          \right.
  \end{align*}
  Together this with \eqref{gamma} and (3.5), we know
  \[
  {\rm Prox}_{\gamma C L_{0/1}}(\bu^*-\gamma\boldsymbol{\lambda}^*)  =\bu^*.
  \]
  The proof is hence complete.
    \end{proof}

% use section* for acknowledgment
\ifCLASSOPTIONcompsoc
  % The Computer Society usually uses the plural form
  \section*{Acknowledgments}
\else
  % regular IEEE prefers the singular form
  \section*{Acknowledgment}
\fi

   The first author is supported in part by the Natural Science Foundation of China under grants 11901595 and 11971490, the Natural Science Foundation of Guangdong Province under grants 2021A1515012345, by Science and Technology Program of Guangzhou (202201010677), and by the Opening Project of Guangdong Province Key Laboratory of Computational Science at the Sun Yat-sen University under grant 2021018. The Third author is supported by Guangdong Basic and Applied Basic Research Foundation (2020A1515010408).

% Can use something like this to put references on a page
% by themselves when using endfloat and the captionsoff option.
\ifCLASSOPTIONcaptionsoff
  \newpage
\fi

% trigger a \newpage just before the given reference
% number - used to balance the columns on the last page
% adjust value as needed - may need to be readjusted if
% the document is modified later
%\IEEEtriggeratref{8}
% The "triggered" command can be changed if desired:
%\IEEEtriggercmd{\enlargethispage{-5in}}

% references section

% can use a bibliography generated by BibTeX as a .bbl file
% BibTeX documentation can be easily obtained at:
% http://mirror.ctan.org/biblio/bibtex/contrib/doc/
% The IEEEtran BibTeX style support page is at:
% http://www.michaelshell.org/tex/ieeetran/bibtex/
%\bibliographystyle{IEEEtran}
% argument is your BibTeX string definitions and bibliography database(s)
%\bibliography{IEEEabrv,../bib/paper}
%
% <OR> manually copy in the resultant .bbl file
% set second argument of \begin to the number of references
% (used to reserve space for the reference number labels box)
%\begin{thebibliography}{1}
%
%\bibitem{IEEEhowto:kopka}
%H.~Kopka and P.~W. Daly, \emph{A Guide to \LaTeX}, 3rd~ed.\hskip 1em plus
%  0.5em minus 0.4em\relax Harlow, England: Addison-Wesley, 1999.
%
%\end{thebibliography}

\bibliographystyle{ieeetr}
\bibliography{KSVM}

\begin{thebibliography}{10}

\bibitem{Scholkopf2001}
B.~Sch\"{o}lkopf and A.~J. Smola, {\em Learning with Kernels: Support Vector
  Machines, Regularization, Optimization, and Beyond}.
\newblock The MIT Press, Cambridge, December 2001.

\bibitem{Steinwart08}
I.~Steinwart and A.~Christmann, {\em Support Vector Machines}.
\newblock Information Science and Statistics, Springer, New York, 2008.

\bibitem{Zaki2020}
M.~J. Zaki and W.~Meira~Jr, {\em Data Mining and Machine Learning: Fundamental
  Concepts and Algorithms}.
\newblock Cambridge University Press, 2020.

\bibitem{Frenay2014}
B.~Frenay and M.~Verleysen, ``Classification in the presence of label noise: A
  survey,'' {\em IEEE Transactions on Neural Networks and Learning Systems},
  vol.~25, no.~5, pp.~845--869, 2014.

\bibitem{Huang2014}
X.~Huang, L.~Shi, and J.~A.~K. Suykens, ``Support vector machine classifier
  with pinball loss,'' {\em IEEE Transactions on Pattern Analysis and Machine
  Intelligence}, vol.~36, no.~5, pp.~984--997, 2014.

\bibitem{Feng2016}
Y.~Feng, Y.~Yang, X.~Huang, S.~Mehrkanoon, and J.~A.~K. Suykens, ``Robust
  support vector machines for classification with nonconvex and smooth
  losses,'' {\em Neural Computation}, vol.~28, no.~6, pp.~1217--1247, 2016.

\bibitem{Wang2022}
H.~Wang, Y.~Shao, S.~Zhou, C.~Zhang, and N.~Xiu, ``Support vector machine
  classifier via $ l_{0/1}$ soft-margin loss,'' {\em IEEE Transactions on
  Pattern Analysis and Machine Intelligence}, 2021.

\bibitem{WangXiu2022}
H.~Wang, Y.~Shao, and N.~Xiu, ``Proximal operator and optimality conditions for
  ramp loss {SVM},'' {\em Optimization Letters}, vol.~16, no.~3, pp.~999--1014,
  2022.

\bibitem{Shen2017}
X.~Shen, L.~Niu, Z.~Qi, and Y.~Tian, ``Support vector machine classifier with
  truncated pinball loss,'' {\em Pattern Recognition}, vol.~68, pp.~199--210,
  2017.

\bibitem{WangShao2022}
H.~Wang, Y.~Shao, and N.~Xiu, ``Proximal operator and optimality conditions for
  ramp loss svm,'' {\em Optimization Letters}, vol.~16, no.~3, pp.~999--1014,
  2022.

\bibitem{WangShao2023}
H.~Wang and Y.~Shao, ``Fast truncated huber loss svm for large scale
  classification,'' {\em Knowledge-Based Systems}, vol.~260, p.~110074, 2023.

\bibitem{Wang2023}
H.~Wang and Y.~Shao, ``Sparse and robust svm classifier for large scale
  classification,'' {\em Applied Intelligence}, 2023.

\bibitem{WangLi2023}
H.~Wang, G.~Li, and Z.~Wang, ``Fast svm classifier for large-scale
  classification problems,'' {\em Information Sciences}, p.~119136, 2023.

\bibitem{Chen1996}
C.~Chen and O.~L. Mangasarian, ``Hybrid misclassification minimization,'' {\em
  Advances in Computational Mathematics}, vol.~5, no.~1, pp.~127--136, 1996.

\bibitem{Cortes1995}
C.~Cortes and V.~Vapnik, ``Support-vector networks,'' {\em Machine Learning},
  vol.~20, no.~3, pp.~273--297, 1995.

\bibitem{Domingos1997}
P.~Domingos and M.~Pazzani, ``On the optimality of the simple bayesian
  classifier under zero-one loss,'' {\em Machine learning}, vol.~29, no.~2,
  pp.~103--130, 1997.

\bibitem{Brooks2011}
J.~P. Brooks, ``Support vector machines with the ramp loss and the hard margin
  loss,'' {\em Operations Research}, vol.~59, no.~2, pp.~467--479, 2011.

\bibitem{Nguyen2013}
T.~Nguyen and S.~Sanner, ``Algorithms for direct 0--1 loss optimization in
  binary classification,'' in {\em International Conference on Machine
  Learning}, pp.~1085--1093, PMLR, 2013.

\bibitem{Tang2018}
J.~Tang, N.~Zhang, and Q.~Li, ``Robust binary classification via
  $\ell_0$-svm,'' in {\em 2018 IEEE International Conference on Data Mining
  Workshops (ICDMW)}, pp.~1263--1270, 2018.

\bibitem{Dhara2011}
A.~Dhara and J.~Dutta, {\em Optimality Conditions in Convex Optimization: A
  Finite-dimensional View}.
\newblock CRC Press, Boca Raton, FL, 2012.

\bibitem{Wendland2005}
H.~Wendland, {\em Scattered Data Approximation}, vol.~17 of {\em Cambridge
  Monographs on Applied and Computational Mathematics}.
\newblock Cambridge: Cambridge University Press, 2005.

\bibitem{1998Variational}
T.~R. Rockafellar and R.~Wets, {\em Variational Analysis}.
\newblock Springer, 1998.

\bibitem{mordukhovich1994generalized}
B.~S. Mordukhovich, ``Generalized differential calculus for nonsmooth and
  set-valued mappings,'' {\em Journal of Mathematical Analysis and
  Applications}, vol.~183, no.~1, pp.~250--288, 1994.

\bibitem{Liu2019}
Y.~Liu and S.~Pan, ``Regular and limiting normal cones to the graph of the
  subdifferential mapping of the nuclear norm,'' {\em Set-Valued and
  Variational Analysis. Theory and Applications}, vol.~27, no.~1, pp.~71--85,
  2019.

\bibitem{Liu2020}
Y.~Liu, S.~Bi, and S.~Pan, ``Several classes of stationary points for rank
  regularized minimization problems,'' {\em SIAM J. Optim.}, vol.~30, no.~2,
  pp.~1756--1775, 2020.

\bibitem{ioffe1979regular}
A.~D. Ioffe, ``Regular points of lipschitz functions,'' {\em Transactions of
  the American Mathematical Society}, vol.~251, pp.~61--69, 1979.

\bibitem{dontchev2004regularity}
A.~L. Dontchev and R.~T. Rockafellar, ``Regularity and conditioning of solution
  mappings in variational analysis,'' {\em Set-Valued Analysis}, vol.~12,
  no.~1, pp.~79--109, 2004.

\bibitem{henrion2005calmness}
R.~Henrion and J.~V. Outrata, ``Calmness of constraint systems with
  applications,'' {\em Mathematical Programming}, vol.~104, no.~2,
  pp.~437--464, 2005.

\bibitem{ioffe2008metric}
A.~D. Ioffe and J.~V. Outrata, ``On metric and calmness qualification
  conditions in subdifferential calculus,'' {\em Set-Valued Analysis}, vol.~16,
  no.~2, pp.~199--227, 2008.

\bibitem{gfrerer2011first}
H.~Gfrerer, ``First order and second order characterizations of metric
  subregularity and calmness of constraint set mappings,'' {\em SIAM Journal on
  Optimization}, vol.~21, no.~4, pp.~1439--1474, 2011.

\bibitem{bai2019directional}
K.~Bai, J.~J. Ye, and J.~Zhang, ``Directional quasi-/pseudo-normality as
  sufficient conditions for metric subregularity,'' {\em SIAM Journal on
  Optimization}, vol.~29, no.~4, pp.~2625--2649, 2019.

\bibitem{Liu2022}
S.~Pan, L.~Liang, and Y.~Liu, ``Local optimality for stationary points of group
  zero-norm regularized problems and equivalent surrogates,'' {\em
  Optimization}, vol.~0, no.~0, pp.~1--33, 2022.

\bibitem{Wu2021}
Y.~Wu, S.~Pan, and S.~Bi, ``Kurdyka-\l ojasiewicz property of zero-norm
  composite functions,'' {\em Journal of Optimization Theory and Applications},
  vol.~188, no.~1, pp.~94--112, 2021.

\bibitem{1999Strong}
H.~H. Bauschke, J.~M. Borwein, and L.~Wu, ``Strong conical hull intersection
  property, bounded linear regularity, jameson's property (g), and error bounds
  in convex optimization,'' {\em Mathematical Programming}, vol.~86, no.~1,
  pp.~135--160, 1999.

\bibitem{boyd2011distributed}
S.~Boyd, N.~Parikh, E.~Chu, B.~Peleato, J.~Eckstein, {\em et~al.},
  ``Distributed optimization and statistical learning via the alternating
  direction method of multipliers,'' {\em Foundations and Trends in Machine
  learning}, vol.~3, no.~1, pp.~1--122, 2011.

\bibitem{Mordukhovich2013}
B.~S. Mordukhovich and N.~M. Nam, {\em An Easy Path to Convex Analysis and
  Applications}.
\newblock Morgan and Claypool Publishers, 2014.

\end{thebibliography}
% biography section
%
% If you have an EPS/PDF photo (graphicx package needed) extra braces are
% needed around the contents of the optional argument to biography to prevent
% the LaTeX parser from getting confused when it sees the complicated
% \includegraphics command within an optional argument. (You could create
% your own custom macro containing the \includegraphics command to make things
% simpler here.)

%%%%%%%%%%%%%%%%%%%%%%%%%%%%%%%%%%%%%%%
\begin{IEEEbiography}[{\includegraphics[width=1in,height=1.25in,clip,keepaspectratio]{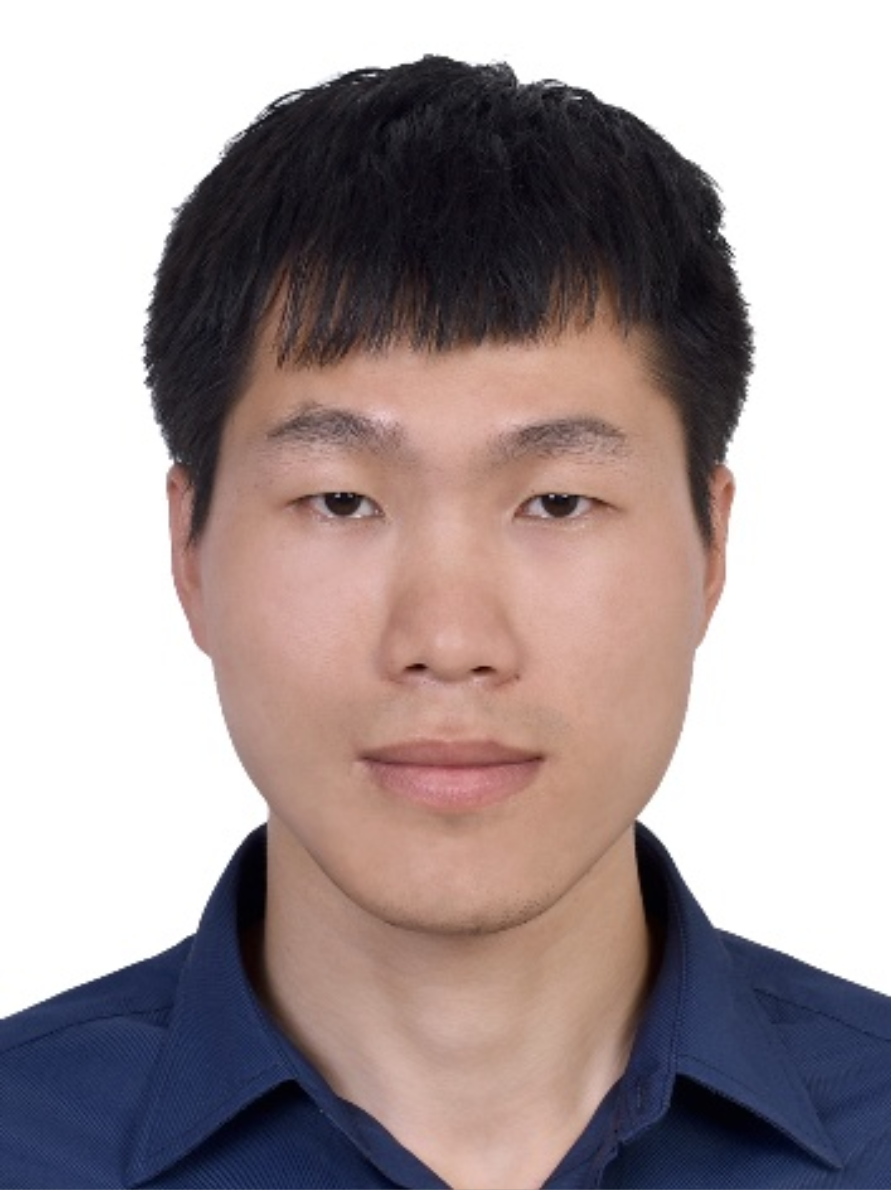}}]{Rongrong Lin}
received the B.Sc. degree and Ph.D. degree in mathematics from Sun Yat-sen University, China,  in 2014 and 2017, respectively. He was a research assistant of the University of Alberta, Canada, from 2015 to 2016.
From 2017 to 2020, he was an Associate Research Fellow with Sun Yat-sen University, China. He was a Visiting Scholar with Old Dominion University, USA in 2018.
Currently, he is a Lecturer at School of Mathematics and Statistics, Guangdong University of Technology, China.  His research interests include machine learning, learning theory, and approximation theory.
\end{IEEEbiography}

\begin{IEEEbiography}[{\includegraphics[width=1in,height=1.25in,clip,keepaspectratio]{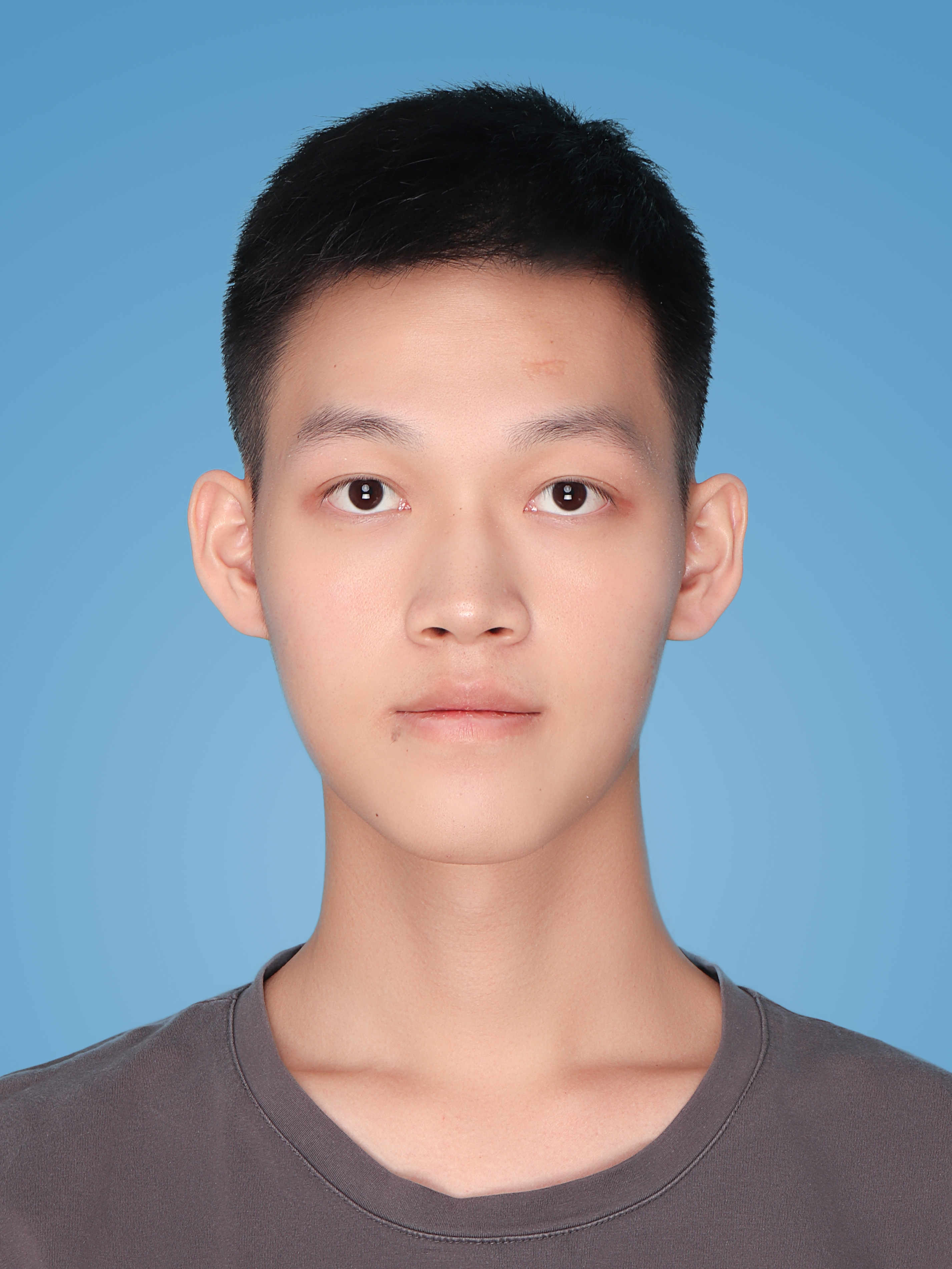}}]{Yingjia Yao}  received the B.Sc. degree from Guangdong University of Finance, China, in 2021.
He is currently a master candidate of School of Mathematics and Statistics, Guangdong University of Technology, China.
His research interests include machine learning and optimization theory.
\end{IEEEbiography}

\begin{IEEEbiography}[{\includegraphics[width=1in,height=1.25in,clip,keepaspectratio]{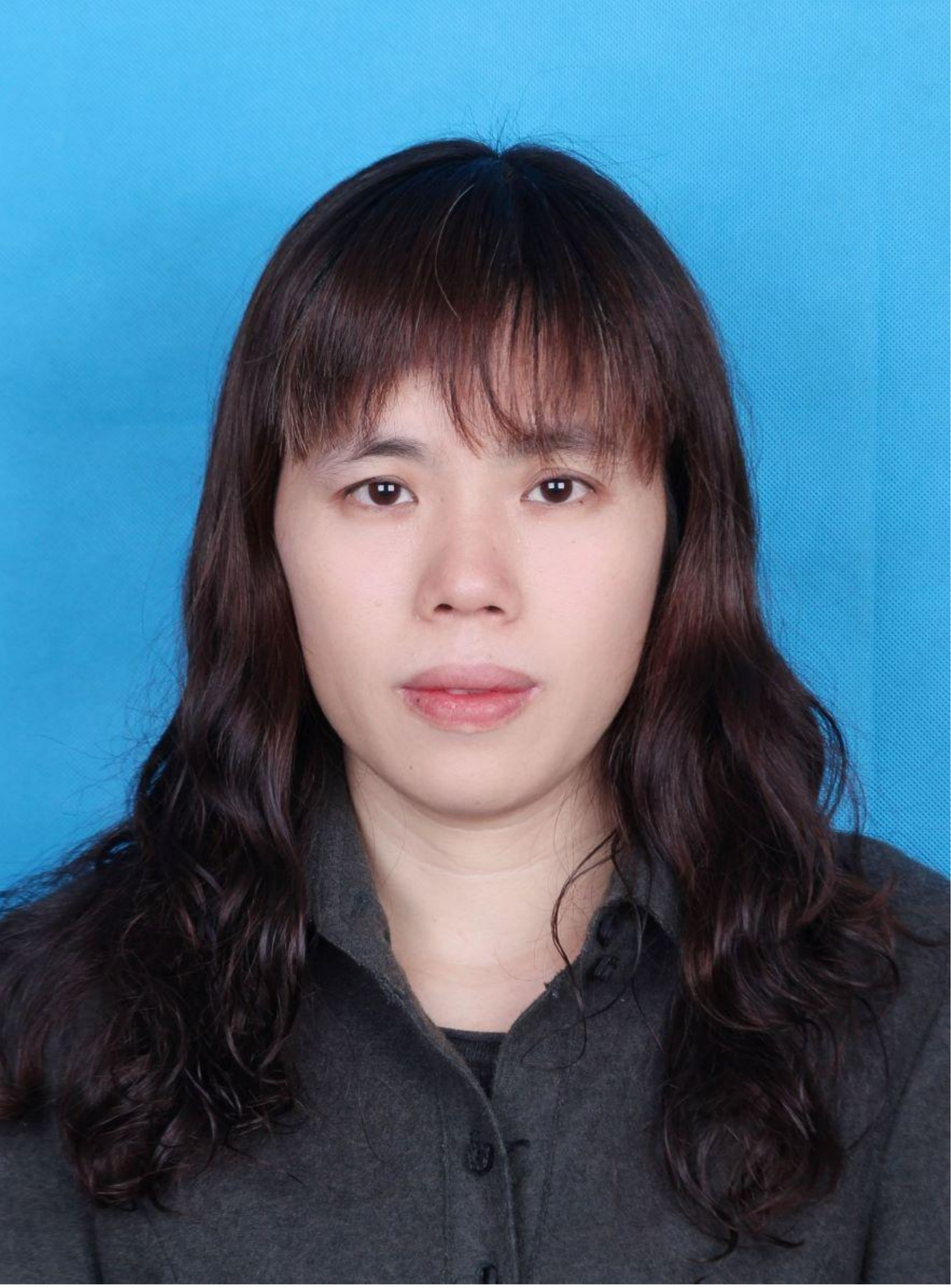}}]{Yulan Liu}
received her M.Sc. degree in Faculty of Mathematics from Nanchang University, Nanchang, China, in 2002, and received Ph.D. degree in School of Mathematics, South China University of Technology, Guangzhou, China, in 2018. Currently, she is an Associate Professor in School of Mathematics and Statistics, Guangdong University of Technology. Her research interests include machine learning and matrix optimization and application.

\end{IEEEbiography}

\end{document}